\documentclass{colt2017}

\usepackage{amsmath,amsfonts,amssymb,dsfont,mathtools}
\usepackage{enumerate}
\usepackage{times}
\usepackage{comment}
\usepackage{soul,color}
\usepackage{graphicx,float,wrapfig}
\usepackage[usenames,dvipsnames]{pstricks}
\usepackage{pst-grad}
\usepackage{pst-plot}
\usepackage[linesnumbered,boxed,ruled,vlined]{algorithm2e}
\usepackage[compact]{titlesec}
\usepackage{indentfirst}


\definecolor{Darkblue}{rgb}{0,0,0.4}
\definecolor{Brown}{cmyk}{0,0.61,1.,0.60}
\definecolor{Purple}{cmyk}{0.45,0.86,0,0}

\newtheorem{theo}{Theorem}[section]

\newtheorem{Lemma}[theo]{Lemma}

\newtheorem{defi}[theo]{Definition}
\newtheorem{rem}[theo]{Remark}

\newenvironment{proofof}[1]{\begin{proof}[Proof of #1]}{\end{proof}}

\renewcommand{\oplus}{\bigtriangleup}

\newcommand{\Ex}{\mathbb{E}}
\newcommand{\pr}[1]{\Pr\left[#1\right]}

\newcommand{\arm}{A}

\newcommand{\alg}{\mathbb{A}}
\newcommand{\event}{\mathcal{E}}
\newcommand{\KL}{\mathrm{KL}}

\newcommand{\Normal}{\mathcal{N}}

\newcommand{\amean}[1]{\mu_{#1}}
\newcommand{\hamean}[1]{\hat{\mu}_{#1}}

\newcommand{\bestkarm}{\textsc{Best-$k$-Arm}\xspace}
\newcommand{\bestarm}{\textsc{Best-$1$-Arm}\xspace}

\newcommand{\Gap}[1]{\Delta_{[#1]}}

\newcommand{\combibandit}{\textsc{Best-Set}\xspace}

\newcommand{\generalbandit}{\textsc{General-Samp}\xspace}

\newcommand{\argmax}{\operatorname*{argmax}}

\newcommand{\eat}[1]{}

\newcommand{\subsetfam}{\mathcal{F}}
\newcommand{\combiband}{\mathcal{C}}

\newcommand{\paraset}{\mathcal{U}}

\newcommand{\simest}{\textsf{SimultEst}}

\newcommand{\gklow}{\textsf{Low}}
\newcommand{\chenlow}{H_C}

\newcommand{\alggen}{\textsf{LPSample}}

\newcommand{\inst}{\mathcal{I}}    
\newcommand{\armseq}{S}    
\newcommand{\ansset}{O}    
\newcommand{\anssetcol}{\mathcal{O}}    
\newcommand{\real}{\mathbb{R}}
\newcommand{\mean}{\mu}
\newcommand{\empmean}{\hat\mu}    
\newcommand{\tilmean}{\nu}    

\newcommand{\candansset}{\hat\ansset}    

\newcommand{\twonorm}[1]{\left\Vert #1 \right\Vert_{2}}
\newcommand{\alt}{\mathsf{Alt}}
\newcommand{\altans}{\alt(\candansset)}    
\newcommand{\ltwoball}{B}

\newcommand{\initOneLiners}{%
    \setlength{\itemsep}{0pt}
    \setlength{\parsep }{0pt}
    \setlength{\topsep }{0pt}
}

\title{Nearly Optimal Sampling Algorithms for Combinatorial Pure Exploration}
\coltauthor{
	\Name{Lijie Chen}\thanks{
		Institute for Interdisciplinary Information Sciences (IIIS), Tsinghua University, Beijing, China. Supported in part by the National Basic Research Program of China grants 2015CB358700, 2011CBA00300, 2011CBA00301, and the National NSFC grant 61632016.
	}
	\Email{chenlj13@mails.tsinghua.edu.cn}\\
	\Name{Anupam Gupta}\thanks{
		Computer Science Department, Carnegie Mellon University, Pittsburgh, USA.
		Supported in part by NSF awards CCF-1536002, CCF-1540541, and CCF-1617790.
	}
	\Email{anupamg@cs.cmu.edu}\\
	\Name{Jian Li}\footnotemark[1]
	\Email{lijian83@mail.tsinghua.edu.cn}\\
	\Name{Mingda Qiao}\footnotemark[1]
	\Email{qmd14@mails.tsinghua.edu.cn}\\
	\Name{Ruosong Wang}\footnotemark[1]
	\Email{wrs13@mails.tsinghua.edu.cn}\\
 }

\begin{document}
\maketitle 	
\begin{abstract}
	We study the combinatorial pure exploration problem
	\combibandit{} in a
	stochastic multi-armed bandit game. 
	In an \combibandit{} instance, we are given $n$ stochastic arms
	with unknown reward distributions, as well as a family $\subsetfam$
	 of feasible subsets over the arms. Let the weight of
	an arm be the mean of its reward distribution. Our goal is to identify the feasible subset in $\subsetfam$ with the
	maximum total weight, using as few samples as possible.
	The problem generalizes the classical best arm
            identification problem and the top-$k$ arm identification
            problem, both of which have attracted significant attention in recent years.
	We provide a novel \emph{instance-wise} lower bound for the sample complexity of the problem,
	as well as a nontrivial sampling algorithm, matching
	the lower bound up to a factor of $\ln|\subsetfam|$.
	For an important class of combinatorial families (including spanning trees, matchings, and path constraints), we also provide 
	polynomial time implementation of the sampling algorithm,
	using the equivalence of separation and optimization for convex program, and 
	the notion of approximate Pareto curves in multi-objective optimization 
	(note that $|\subsetfam|$ can be exponential in $n$). 
	We also show that the $\ln|\subsetfam|$ factor is inevitable 
	in general, through a nontrivial lower bound construction
	utilizing a combinatorial structure resembling the Nisan-Wigderson design.
	Our results significantly improve several previous results
	for several important combinatorial constraints, 
	and provide a tighter understanding of the general \combibandit problem.
	
	We further introduce an even more general problem,
	formulated in geometric terms. 
	We are given $n$ Gaussian arms with unknown means and unit variance.
	Consider the $n$-dimensional 
	Euclidean space $\mathbb{R}^n$, and a collection
            $\anssetcol$ of disjoint subsets. Our goal is to
            determine the subset in $\anssetcol$ that contains the mean profile (which is the $n$-dimensional vector of the means), using
	as few samples as possible.
	The problem generalizes most pure exploration bandit problems studied in the literature. 
	We provide the first nearly optimal sample complexity upper and lower bounds 
	for the problem.
\end{abstract}

\section{Introduction}
The stochastic multi-armed bandit model is a
classical model for characterizing the exploration-exploitation tradeoff in a variety of application fields with stochastic environments. 
In this model, we are given a set of $n$ stochastic arms,
each associated with an unknown reward distribution. Upon each play 
of an arm, we can get a reward sampled from the corresponding distribution.
The most well studied objectives include maximizing the cumulative sum of rewards, or minimizing the cumulative regret (see e.g., \cite{cesa2006prediction,bubeck2012regret}).
Another popular objective is to identify the optimal solution (which can
either be a single arm, or a set of arms, depending on the problem) with
high confidence, using as few samples as possible. This problem is
called the {\em pure exploration} version of the multi-armed bandit
problem, and has attracted significant attention due to applications in
domains like medical trials, crowdsourcing, communication network,
databases and online advertising \cite{chen2014combinatorial,
  zhou2014optimal, cao2015top}. 

The problems of identifying the best arm (i.e., the arm with maximum expected reward) and the top-$k$ arms have been studied extensively (see e.g., \cite{mannor2004sample,even2006action,audibert2010best,kalyanakrishnan2010efficient,gabillon2012best,kalyanakrishnan2012pac,karnin2013almost,jamieson2014lil,zhou2014optimal,chen2015optimal,cao2015top,carpentier2016tight,garivier2016optimal,chen2016towards,chen2017nearly}).
\cite{chen2014combinatorial} proposed the following 
significant generalization, in which the cardinality constraint is replaced by a general combinatorial constraint and the goal
is to identify the best subset (in terms of the total mean) satisfying the constraint. 

\begin{defi}[\combibandit]\label{defi:comb-bandit-PE}
	In a \combibandit\ instance $\combiband=(S,\subsetfam)$,
	we are given a set $S$ of $n$ arms.
	Each arm $a \in S$ is associated with 
  a Gaussian reward distribution with unit variance and an unknown mean $\mu_a$.
	We are also given a family of subsets $\subsetfam$ with ground set identified with the set $S$ of arms.
	Our goal is to find with probability at least $1-\delta$, a subset in $\subsetfam$ with the maximum total mean
	using as few samples as possible.
  We assume that there is a unique subset with the maximum total mean.
\end{defi}

In the above definition, the set family $\subsetfam$ may be given explicitly
(i.e., the list of all subsets in $\subsetfam$ is given as input),
or implicitly in the form of some combinatorial constraint (e.g., matroids, matchings, paths).
Note that in the latter case, $|\subsetfam|$
may be exponential in the input size; we would additionally like to design sampling algorithms that run in
polynomial time. Some common combinatorial constraints are the following:
\begin{enumerate}
\item (\textsc{Matroids}) $(S,\subsetfam)$ is a matroid, where $S$ is
  the ground set and $\subsetfam$ is the family of independent set of
  the matroid. The problem already captures a number of interesting
  applications. See~\cite{chen2016pure} for more details and the
  state-of-the-art sample complexity bounds.

\item (\textsc{Paths}) Consider a network $G$, in which the latency of
  each edge is stochastic. However, the distributions of the latencies
  are unknown and we can only take samples.  We would like to choose a
  path from node $s$ to node $t$ such that the expected latency is
  minimized.  Here $S$ is the set of edges of a given undirected graph
  $G$, and $\subsetfam$ the set of $s$-$t$ paths in $G$.
	 
\item (\textsc{Matchings}) There are $n$ workers and $n$ types of jobs.
  Each job type must be assigned to exactly one worker, and each worker
  can only handle one type of job. Jobs of the same type may not be
  exactly the same, hence may have different profit. For simplicity, we
  assume that for worker $i$, the profit of finishing a random job in
  type $j$ follows an unknown distribution $D_{ij}$.  Our goal is to
  find an assignment of types to workers, such that the expected total
  reward is maximized (assuming each worker gets a random job from the
  type assigned to them). This problem has potential applications to
  crowdsourcing.  Here, $S$ is the set of edges in the worker-job-type
  bipartite graph $G$, and $\subsetfam$ the set of perfect matchings
  in~$G$.

\item (\textsc{Tree-Planning}) We are given a tree, where each arm $i$
  corresponds to an edge of the tree. The family $\subsetfam$ is the set
  of paths from the root to the leaves. The goal is to find the
  root-leaf path with the maximum weight.  This setting corresponds to
  the open-loop planning problem of maximizing the expected sum of
  rewards over consecutive actions from a starting state (i.e., the
  root) when the state dynamics is deterministic and the reward
  distributions are unknown (see e.g., \cite{munos2014bandits,
    gabillon2016improved}).
\end{enumerate}

While these examples show that the \combibandit\ problem is quite general,
there are other interesting pure exploration bandit problems that cannot
be captured by such combinatorial constraints over the arm set.  For
example, suppose there are $n$ Gaussian arms with unknown means and unit
variance. We know there is exactly one special arm with mean in the
interval $(0.4,0.6)$; all other arms have means either strictly larger
than 0.6, or strictly less than 0.4.  We would like to identify this
special arm.  To this end, we define a general sampling problem, which
captures such problems, as follows.

\begin{defi}[\generalbandit]
  An instance of the general sampling problem is a
  pair $\inst = (\armseq,\anssetcol)$, where $\armseq = (\arm_1,
  \arm_2,\ldots,\arm_n)$ is a sequence of $n$ Gaussian arms each with
  unit variance, and $\anssetcol$ is a collection of answer sets, each of
  which is a subset of $\real^n$. Let $\mean_i$ denote the mean of arm
  $\arm_i$. The vector $\mean$ is called the mean profile of the
  instance. In each round, we choose one of the arms and obtain an
  independent sample drawn from its reward distribution. The goal is to
  identify with probability $1- \delta$ the unique set in $\anssetcol$
  that contains $\mean$, while using as few samples as possible.  It is
  guaranteed that $\mean\in\bigcup_{\ansset\in
    \anssetcol}\ansset$, and for each $\ansset\in\anssetcol$,
  the closure of $\bigcup_{\ansset' \in
    \anssetcol\setminus\{\ansset\}}\ansset'$ is disjoint from $\ansset$.
\end{defi}

This definition of the general sampling problem captures well-studied
bandit problems (with unique solutions) in the pure-exploration
setting:
\begin{enumerate}
\item In the best arm identification problem, $\anssetcol$ contains
  exactly $n$ answer sets, where the $i$-th answer set is given by
  $\{\mean\in\real^{n}:\mean_i>\max_{j\ne i}\mean_j\}$.
\item In the \combibandit\ problem
  (Definition~\ref{defi:comb-bandit-PE}), $\anssetcol$ contains exactly
  $|\subsetfam|$ answer sets, where each answer set corresponds to a set
  $S\in \subsetfam$, and is given by $\{\mean\in\real^{n}:\sum_{i\in
    S}\mean_i>\sum_{i\in T}\mean_i, \forall T\in \subsetfam, T\ne S\}$.
\item There are $n$ Gaussian arms with unknown means and unit
  variance. We would like to find how many arms have mean larger than a
  given threshold $\theta$.  
  This is a variant of the threshold bandit problem (see e.g., \cite{locatelli2016optimal}).
  $\anssetcol$ contains exactly $n$ answer
  sets, where the $j$-th answer set is given by
  $\{\mean\in\real^{n}:\sum_{i\in [n]}\mathbb{I}\{\mean_i>\theta\}=j\}$.
  We assume that no arm has mean exactly $\theta$ (to guarantee
  disjointness).
\item There are $n$ Gaussian arms with unknown means and unit variance.
  Given a threshold $\theta$, we want to determine whether the span
  (i.e., the difference between the largest and smallest means) is
  greater than $\theta$.  We assume that no difference of two arms is
  exactly $\theta$ (to guarantee disjointness).
\item Consider a zero-sum game in which each player has $K$ available actions.
  If the two players choose actions $i\in[K]$ and $j\in[K]$,
  they receive rewards $a_{i,j}$ and $-a_{i,j}$, respectively.
  We want to find the \emph{maximin} action of the first player,
  i.e., $\argmax_{i\in[K]}\min_{j\in[K]}a_{i,j}$,
  using noisy queries on the rewards ($a_{i,j}$'s).
  This is similar to the Maximin Action Identification problem~\cite{garivier2016maximin}.
  Here there are $n = K ^ 2$ arms with means $\mu_{(i-1)n+j} = a_{i,j}$.
  The answer sets are given by $\anssetcol = \{\ansset_1, \ldots, \ansset_K\}$,
  where \[\ansset_i = \left\{\mu\in\real^{n}:\min_{j\in[K]}\mu_{(n-1)i+j} > \max_{k\in[K]\setminus\{i\}}\min_{j\in[K]}\mu_{(n-1)k+j}\right\}.\]
\end{enumerate}
	
\begin{rem}
  The disjointness requirement, that the closure of
  $\bigcup_{\ansset'\in\anssetcol}\ansset'$ is disjoint from $\ansset$
  for any $\ansset\in\anssetcol$, is crucial to the solvability of the
  instance. For example, no $\delta$-correct algorithm (for some
  $\delta<0.5$) can solve the instance with a single arm with zero mean
  and $\anssetcol = \{(-\infty,0),[0,+\infty)\}$ within a finite number
  of samples in
  expectation.
  Furthermore, the disjointness condition guarantees that the correct
  solution is unique. Hence, our problem cannot capture some PAC
  problems in which there may be many approximate solutions.
\end{rem}

\begin{rem}
  Our problem is closely related to the active multiple hypothesis
  testing problem.  See the related work section for more discussions.
\end{rem}

\subsection{Our Results}

In order to formally state our results, we first define the notion of $\delta$-correct algorithms.
\begin{defi}[$\delta$-correct algorithms]\label{defi:delta-correct}
  We say Algorithm $\alg$ is a $\delta$-correct algorithm for
  \combibandit if on every instance $\combiband = (S, \subsetfam)$,
  algorithm $\alg$ identifies the set in $\subsetfam$ with the largest
  total mean with probability at least $1 - \delta$.
  
  Similarly, we say Algorithm $\alg$ is a $\delta$-correct algorithm for
  \generalbandit if on every instance $\inst = (S, \anssetcol)$,
  algorithm $\alg$ identifies the set in $\anssetcol$ which contains the mean profile of the instance with probability at least $1 - \delta$.
\end{defi}

\subsubsection{Instance Lower Bound via Convex Program}
\label{sec:lb}
\cite{garivier2016optimal} obtained a strong
lower bound for the sample complexity of $\bestarm$ based on the change
of distribution. Their lower bound is in fact the solution of a
mathematical program. They show that for $\bestarm$, one can derive the
explicit solution for several distributions.

Garivier and Kaufmann's approach is general and can be applied to
$\combibandit$ and $\generalbandit$ as well. However, the resulting
mathematical program is not easy to work with. Unlike $\bestarm$, we
cannot hope for an explicit solution for the general $\combibandit$
problem: the program has an infinite number of constraints, and it is
unclear how to solve it computationally.  Instead, we adopt their
framework and derive an equivalent convex program for $\combibandit$ (in
Section~\ref{sec:comb-lowerb}). Using this, we obtain the following
result.
\begin{theo}\label{theo:LB}
	Let $\combiband = (S,\subsetfam)$ be an instance of $\combibandit$. 
	Let $\gklow(\combiband)$ be the optimal value of the convex program~\eqref{eq:gklowdef}
  (see Section~\ref{sec:comb-lowerb}). 
	Then for any $\delta \in (0,0.1)$
  and $\delta$-correct algorithm $\alg$ for $\combibandit$,
	$\alg$ takes
    $\Omega(\gklow(\combiband) \ln \delta^{-1})$ 
	samples in expectation on $\combiband$.
\end{theo}

Our new lower bound has the following computational advantage.  First,
it is a solution of a convex program with a finite number of
constraints.  Hence, one can solve it in time polynomial in $n$ and the
number of constraints (note that there may be exponential many of them, if
$|\subsetfam|$ is exponentially large).  Moreover, for some important
classes of $\subsetfam$, we can approximate the optimal value of the
convex program within constant factors in polynomial time (see
Section~\ref{sec:efficomp} for more details).

\paragraph{Comparison with the Lower Bound in \cite{chen2014combinatorial}.} 
Let $\combiband = (S,\subsetfam)$ be an instance of $\combibandit$ with
the optimal set $O\in\subsetfam$. Assume that all arms are Gaussian with
unit variance.  It was proved in~\cite{chen2014combinatorial} that for
any $\delta \in (0, e^{-16} / 4)$ and any $\delta$-correct algorithm
$\mathbb{A}$, $\mathbb{A}$ takes $ \Omega\left(\chenlow(\combiband)
  \ln\delta^{-1} \right) $ samples in expectation. Here
$\chenlow(\combiband) = \sum_{i\in S}\Delta_i^{-2}$ is the {\em
  hardness} of the instance $\combiband$, where $\Delta_i$, the
\textit{gap} of arm $i\in S$, is defined as
$$
\Delta_i = \begin{cases}
\mu(O) - \max_{O'\in\subsetfam, i\notin O'}\mu(O'), & i \in O,\\
\mu(O) - \max_{O'\in\subsetfam, i\in O'}\mu(O'), & i \notin O\text{.}
\end{cases}
$$

We can show that our lower bound is no
weaker than the lower bound in \cite{chen2014combinatorial}.
\begin{theo}\label{lem:vschenlow}
Let $\combiband = (S,\subsetfam)$  be an instance of $\combibandit$,
$$
\gklow(\combiband) \ge H_{C}(\combiband).
$$
\end{theo}

The proof of Theorem~\ref{lem:vschenlow} can be found in Section \ref{sec:vschenlow}.

\newcommand{\combibandhardUCB}{\combiband_{\mathsf{disj\text{-}sets}}}

Furthermore, we note that for certain instances of \textsc{Matchings}
and \textsc{Paths}, our lower bound can be stronger than the
$\chenlow(\combiband)\ln\delta^{-1}$ bound by an $\Theta(n)$ factor.
We consider the following simple instance $\combibandhardUCB$ of
\combibandit{} that consist of $n = 2k$ arms numbered $1$ through $n$.
The only two feasible sets are $A=\{1,2,\ldots,k\}$ and
$B=\{k+1,k+2,\ldots,2k\}$ (i.e., $\subsetfam=\{A,B\}$).  The mean of
each arm in $A$ is $\epsilon>0$, while each arm in $B$ has a mean of
$0$.  A simple calculation shows that $\gklow(\combibandhardUCB) =
\Theta(\epsilon^{-2})$, while $\chenlow(\combibandhardUCB) =
O\left(\epsilon^{-2}/n\right)$.  This establishes an $\Omega(n)$ factor
separation.

Indeed, $\combibandhardUCB$ is a special case of many $\combibandit$
instances, including \textsc{Matchings} (consider a cycle with length
$2k$; there are two disjoint perfect matchings) and
\textsc{Paths} (consider a graph with only two disjoint $s$-$t$ paths of
length $k$).  Thus, understanding the complexity of $\combibandhardUCB$
is crucial for understanding more complicated instances.

\subsubsection{Positive Result I: A Nearly Optimal Algorithm for $\combibandit$}
\label{sec:alg}
Our first positive result is a nearly optimal algorithm for
$\combibandit$.
\begin{theo}\label{theo:naive-upperb}
  There is a $\delta$-correct algorithm for $\combibandit$ that takes
  $$
  O\left( \gklow(\combiband)\ln\delta^{-1} +
    \gklow(\combiband)\ln\Delta^{-1}\left(\ln\ln\Delta^{-1} +
      \ln|\subsetfam|\right) \right)
  $$
  samples on an instance $\combiband = (S,\subsetfam)$ in expectation,
  where $\Delta$ denote the gap between the set with the second largest
  total mean in $\subsetfam$ and the optimal set $\ansset$.
\end{theo}

\paragraph{Comparison with Previous Algorithms.}
Again, consider the instance $\combibandhardUCB$, which consists of $k$
arms with mean $\epsilon > 0$ and another $k$ arms with mean zero.  A
straightforward strategy to determine whether $A$ or $B$ has a larger
total mean is to sample each arm $\tau = 8/(k\epsilon^2)\ln(2/\delta)$
times, and determine the answer based on the sign of
$\hat\mu(A)-\hat\mu(B)$.  Lemma~\ref{lem:sum_dev} implies that with
probability at least $1-\delta$, $\hat\mu(A)-\hat\mu(B)$ lies within an
additive error $k\epsilon/2$ to $\mu(A)-\mu(B)=k\epsilon$.  Hence, we
can identify $A$ as the correct answer using $n\tau =
O(\epsilon^{-2}\ln\delta^{-1})$ samples.

\newcommand{\wid}{\mathrm{width}}

\cite{chen2014combinatorial} developed a
$\delta$-correct algorithm $\mathsf{CLUCB}$ for $\combibandit$ with sample complexity
$$O\left(\wid(\combiband)^2\chenlow(\combiband)\ln(n\chenlow(\combiband)/\delta)\right)\text{,}$$
where $\wid(\combiband)$ is the \textit{width} of the underlying combinatorial structure $\subsetfam$ as defined in~\cite{chen2014combinatorial}.  
Hence, roughly speaking, ignoring logarithmic factors, their upper bound
is a $\wid(\combiband)^2$-factor (which is at most $n^2$ factor) away from the complexity
term $\chenlow(\combiband)$ 
they define,\footnote{
	In view of Theorems \ref{theo:LB}~and~\ref{lem:vschenlow},
	$\chenlow(\combiband)\ln (1/\delta)$ is indeed a lower bound.
	}
while our upper bound is at most $\ln |\subsetfam|$ (which is at most $n$) factor away
from our lower bound. Theorem~\ref{theo:worst-case-lowb-comb}
shows that the $\ln |\subsetfam|$ term is also inevitable in the worst case.

In fact, consider the simple instance $\combibandhardUCB$ we defined earlier.
A simple calculation shows that $\wid(\combibandhardUCB)=\Theta(n)$ and $\chenlow(\combibandhardUCB)=\Theta(\epsilon^{-2}/n)$, so $\mathsf{CLUCB}$ requires
$\Omega(n\epsilon^{-2}\ln\delta^{-1})$ samples in total on the simple
instance $\combibandhardUCB$ we defined earlier.  Moreover, a recent
algorithm proposed in \cite{gabillon2016improved} also takes
$\Omega(n\epsilon^{-2}\ln\delta^{-1})$ samples.  In comparison, our
algorithm achieves a sample complexity of
$O(\epsilon^{-2}(\ln\delta^{-1} + \ln\epsilon^{-1}
\ln\ln\epsilon^{-1}))$ on $\combibandhardUCB$, which is nearly optimal. 
Therefore, our algorithm
obtain a significant speed up of order $\Omega(n)$ on certain cases
comparing to all previous algorithms.

In fact, both these previous algorithms for \combibandit are UCB-based,
i.e., they maintain a {\em confidence bound} for each {\em individual
  arm}.  We observe that 
such UCB-based algorithms are inherently inadequate to achieve the optimal
sample complexity, even for the simple instance $\combibandhardUCB$.
Note that in order for a UCB-based algorithm to identify $A$ as the
correct answer for $\combibandhardUCB$, it requires an $O(\epsilon)$
estimation of the mean of each arm in $S$, which requires
$\Omega(n\epsilon^{-2}\ln\delta^{-1})$ samples in total.  Thus, the
sample complexity of previous algorithms based on maintaining confidence
bounds for individual arms is at least a factor of $n$ away from the
optimal sample complexity, even for very simple instances such as
$\combibandhardUCB$.

\paragraph{The $\ln |\subsetfam|$ term is necessary in the worst case:} Note that the sample complexity of our algorithm involves a $\ln |\subsetfam|$ term, which could be large when $\subsetfam$ is exponential in $n$. Hence, it is natural to ask whether one can get rid of it. We show that this is impossible by proving a worst-case lower bound for \combibandit in which the factor $\ln |\subsetfam|$ is necessary.

\begin{theo}\label{theo:worst-case-lowb-comb}
(i) For $\delta \in (0,0.1)$, two positive integers $n$ and $m \le 2^{c n}$ where $c$ is a universal constant,, 
and every $\delta$-correct algorithm $\alg$ for \combibandit, there exists an infinite sequence of $n$-arm instances $\combiband_1=(S_1,\subsetfam_1),\combiband_2=(S_2,\subsetfam_2),\dotsc,$ such that $\alg$ takes at least
$$
\Omega( \gklow(\combiband_k) \cdot (\ln \delta^{-1}+\ln |\subsetfam_k|))
$$
samples in expectation on each $\combiband_k$, $|\subsetfam_k| = m$ for all $k$, and $\gklow(\combiband_k)$ approaches to infinity
as $k\rightarrow +\infty$ . 

\noindent
(ii) Moreover, for each $\combiband_k$ defined above, there exists a $\delta$-correct algorithm $\alg_k$ for \combibandit such that $\alg_k$ takes
$$
O(\gklow(\combiband_k) \cdot \operatorname{poly}(\ln n,\ln\delta^{-1}))
$$
samples in expectation on it.
(The constants in $\Omega$ and $O$ do not depend on $n,m,\delta$ and $k$.)
\end{theo}

	The second part of the above theorem 
	implies that $\gklow(\combiband_k) \cdot \ln\delta^{-1}$ is achievable by some specific $\delta$-correct algorithms for \combibandit (up to $\operatorname{polylog}$ factors). 
	However, the first part states that no matter what algorithm to use,
	one has to pay such a $\ln|\subsetfam|$ factor, which can be as large as $n$,
	for infinitely many instances. 
	Therefore, Theorem~\ref{theo:worst-case-lowb-comb} indeed captures 
	a huge separation between the instance-wise lower bound and the worst-case lower bound:
	they may differ by a large factor of $\ln|\subsetfam|$.

\begin{rem}
	Such a separation is a delicate issue in pure exploration 
	multi-armed bandit problems. Even in the \bestarm\ problem with only two arms,
	such a separation of $\ln\ln \Delta^{-1}$ ($\Delta$ is the gap of the two arms) factor 
	is known (see \cite{chen2015optimal} for more details).
\end{rem}

	We note that  \cite{garivier2016optimal}
	obtained an algorithm for $\bestarm$ that matches the instance lower
	bound as $\delta$ approaches zero.  
	Such algorithms are called {\em asymptotically optimal} in the sequential 
	hypothesis testing literature.
	Their algorithm can potentially be
	adapted to obtain asymptotically optimal algorithms for \combibandit (and even \generalbandit{}) as well.
	From both Theorem~\ref{theo:naive-upperb} and~\ref{theo:worst-case-lowb-comb},
	we can see the sample complexity typically consists of two terms, one
	depending on $\ln 1/\delta$ and one independent of $\delta$
	(this is true for many other pure exploration bandit problems).
	In \cite{garivier2016optimal}, the authors did not investigate the
	the second term in the sample complexity (which does not depend on $\delta$), 
	since it is treated as a ``constant'' (only $\delta$ is treated as a variable and 
	approaches to 0). 
	However, as Theorem~\ref{theo:worst-case-lowb-comb} indicates,
	the ``constant'' term can be quite large (comparing with the first term
	for moderate $\delta$ values, say $\delta=0.01$) 
	and cannot be ignored, especially in 
	our general \combibandit\ problem.
	Hence, in this paper, we explictly pin down the ``constant'' terms for
	our algorithms, and make progress towards tighter bounds on these terms
	(Theorems~\ref{theo:worst-case-lowb-comb} and
	\ref{theo:lower-bound-general}). 

\subsubsection{Positive Result II: An Efficient Implementation for
  Implicit  $\subsetfam$}

One drawback for our first algorithm is that it needs to take full description of $\subsetfam$. So it would become computationally intractable 
if $\subsetfam$ is given {\em implicitly}, and of exponential size. 
Our second algorithm addresses this computational problem in several important special cases, assuming that the underlying combinatorial structure, $\subsetfam$, 
admits an \textit{efficient maximization oracle} and a \textit{pseudo-polynomial algorithm for the exact version}, which we explain next.

We say that a family of \combibandit{} instances $\{\combiband_k\} =
\{(S_k, \subsetfam_k)\}$ has an efficient maximization oracle, if there
is an algorithm that, given a weight function $w$ defined on $S_k$, identifies
the maximum weight set in $\subsetfam_k$ (i.e.,
$\argmax_{A\in\subsetfam_k}\sum_{i\in A}w_i$), and runs in polynomial
time (with respect to $|S_k|$).  Moreover, a family of \combibandit{}
instances $\{\combiband_k\}$ admits a pseudo-polynomial algorithm for
the exact version, if an algorithm, given an integer weight function $w$ on $S_k$
together with a target value $V$, decides whether $\subsetfam_k$
contains a set with total weight \emph{exactly} $V$, and runs in pseudo-polynomial time
(i.e., polynomial with respect to $|S_k|$ and $V$).

\begin{rem}
  The following problems admit efficient maximization oracles and
  pseudo-polynomial algorithms for the exact version:
    \begin{enumerate}
      \setlength\itemsep{0em}
        \item Maximum weight spanning tree.
        \item Maximum weight bipartite matching.
        \item Shortest $s$-$t$ path with non-negative
          weights.~\footnote{
            For the shortest path problem, we need an efficient minimization oracle, which is equivalent.            
            }
    \end{enumerate}
\end{rem}

\begin{theo}\label{theo:effi-comb}
  For any family of \combibandit{} instances that admits efficient maximization oracles and pseudo-polynomial algorithms, there is a $\delta$-correct algorithm for this instance family that takes
  $$
  O\left(
  \gklow(\combiband)\ln\delta^{-1} + \gklow(\combiband)\ln^2\Delta^{-1}\left(\ln\ln\Delta^{-1} + \ln|\subsetfam|\right)
  \right)
  $$
  samples on an instance $\combiband = (S,\subsetfam)$ in expectation, where $\Delta$ denote the gap between the set with the second largest total mean in $\subsetfam$ and the optimal set $\ansset$. Moreover, the algorithm runs in polynomial time with respect to the expected sample complexity.
\end{theo}

\subsubsection{Positive Result III: Nearly Optimal Algorithm for \generalbandit}
Last but not the least, we present an nearly optimal algorithm for \generalbandit.
\begin{theo}\label{theo:general-upperb}
  There is a $\delta$-correct algorithm for $\generalbandit$ that takes
  $$
  O(\gklow(\inst) \; (\ln\delta^{-1} + n^3 + n\ln\Delta^{-1}))
  $$
  samples on any instance $\inst = (\armseq, \anssetcol)$ in expectation, where
    $$\Delta = \inf_{\tilmean\in\alt(\ansset)}\twonorm{\mean-\tilmean}$$
  is defined as the minimum Euclidean distance between the mean profile $\mean$ and an alternative mean profile $\tilmean\in\alt(\ansset)$ with an answer other than $\ansset$.
\end{theo}

\paragraph{Another worst-case lower bound with a factor of $n$.} Note that the
sample complexity of our algorithm above involves a term depending on
$n$. Note that \combibandit is a special case of \generalbandit, and
setting $m = 2^{\Omega(n)}$ in Theorem~\ref{theo:worst-case-lowb-comb},
we obtain an $O(\gklow(\inst) \cdot (n +　\ln\delta^{-1}))$ worst-case lower bound for
\generalbandit. \footnote{We
   observe that $\gklow(\inst)$ is essentially equivalent to
  $\gklow(\combiband)$ when $\inst$ is identical to a \combibandit
  instance~$\combiband$.} 
Therefore, at least we cannot get rid of the dependence
on $n$.
Moreover, the instance behind the lower bound 
in Theorem~\ref{theo:worst-case-lowb-comb}
has an exponentially large $|\anssetcol|$.
So one may wonder if it is possible to reduce the factor $n$ to 
$\ln |\anssetcol|$.
\footnote{
	This is possible for \combibandit, because of Theorem~\ref{theo:naive-upperb}.
	} 
We present another lower bound showing it is also
impossible, by constructing hard instances with $|\anssetcol| = O(1)$.
Our lower bound instances reveal another source of hardness,
different from the combinatorics used in Theorem~\ref{theo:worst-case-lowb-comb}.

\begin{theo}~\label{theo:lower-bound-general}
	For $\delta \in (0,0.1)$, a positive integer $n$ and every $\delta$-correct algorithm $\alg$ for the general sampling problem, there exists an infinite sequence of $n$-arm instances $\inst_1=(S_1,\anssetcol_1),\inst_2=(S_2,\anssetcol_2),\dotsc,$ such that $\alg$ takes at least
	$$
	\Omega( \gklow(\inst_k) \cdot (\ln \delta^{-1}+ n))
	$$
	samples in expectation on each $\inst_k$, $|\anssetcol_k| =O(1)$ for all $k$, and $\gklow(\inst_k)$ goes to infinity. Moreover, for each $\inst_k$, there exists a $\delta$-correct algorithm $\alg_k$ for \generalbandit such that $\alg_k$ takes
	$$
	O(\gklow(\inst_k) \cdot \ln\delta^{-1})
	$$
	samples in expectation on it.
	(The constants in $\Omega$ and $O$ do not depend on $n,m,\delta$ and $k$.)
\end{theo}

\subsection{Our Techniques}

\subsubsection{Overview of Our $\combibandit$ Algorithm}
Our algorithm is based on a process of successive elimination. 
However, unlike  previous approaches~\cite{karnin2013almost,chen2015optimal,chen2016pure,chen2016towards,gabillon2016improved} which maintained a set of arms, our algorithm maintains a collection of candidate sets and performs the eliminations on them.
The goal of the $r$-th round is to eliminate those sets  with a optimality gap\footnote{The optimality gap for a set $A 
			\in \subsetfam$ is simply $\mu(O) - \mu(A)$ where $O$ is the optimal set.} of at least $\Theta(2^{-r})$.

In order to implement the above elimination, we adopt a mathematical program,
which is a nontrivial modification of the one in Theorem~\ref{theo:LB} and sample the arms accordingly to obtain an $2^{-r}$ approximation of the gap between \emph{every pair} of sets that are still under consideration.
If there is a set pair $(A, B)$ such that the empirical mean of $B$ exceeds that of $A$ by $2^{-r}$, we are certain that $A$ is not the optimal set, and thus we stop considering $A$ as a candidate answer.
This process is repeated until only one set remains.

\subsubsection{Overview of Our $\combibandit$ Algorithm with Efficient Computation}
The previous algorithm maintains the sets still under consideration at the beginning of each round $r$. 
As the number of feasible sets is typically exponential in the number of arms, 
it may be computationally expensive to maintain these sets explicitly. 
The key to computational efficiency is to find a compact representation of the sets still under consideration. 
Here, we represent these sets by using the empirical means 
and some carefully chosen threshold.

To efficiently solve the mathematical program (which is actually a {\em convex program}) mentioned above,
we apply the Ellipsoid method and use the $\varepsilon$-approximate Pareto curve framework of~\cite{papadimitriou2000approximability}
to design an efficient separation oracle. 
This technique
allows us to approximately solve the convex program in polynomial-time with respect
to the input size and the sample complexity of our algorithm.

\subsubsection{Overview of Our \generalbandit{} Algorithm}
	Our $\generalbandit$ algorithm follows a ``explore-verify'' approach.
	In the first stage of algorithm (exploration stage),
	we sample each arm repeatedly in round-robin fashion, until the confidence
	region of the mean profile $\mean$ intersects exactly one answer set,
	which we identify as the candidate answer.
	The second stage (verification stage) is devoted to verifying the candidate answer.
	To this end, we formulate the optimal sampling profile as a linear program.
	Then, we verify the candidate answer by sampling the arms according to the sampling profile.

	Note that in the exploration stage, the arms are sampled in an inefficient
	round-robin fashion, while the candidate answer is verifed
	using the optimal sampling profile in the second stage.
	Hence, we use a less stringent confidence in Stage 1, and then adopts
	the required confidence level $\delta$ in the second stage.

\subsubsection{Other Technical Highlights}

\paragraph{The factor $\ln |\subsetfam|$ is necessary for \combibandit.} 
In order to establish the worst-case $\gklow(\combiband_k) \cdot \ln
|\subsetfam|$ lower bound (the first part of
Theorem~\ref{theo:worst-case-lowb-comb}), we construct a family
$\subsetfam$ of subsets of $[n]$ satisfying the following two important
properties \footnote{These two properties resemble the well-known set
  system of~\cite{nisan1994hardness}.}:
\begin{itemize}
\item (Sets in $\subsetfam$ are large) Each subset $A \in \subsetfam$ is
  of the same size $\ell = \Omega(n)$.
	
\item (Intersections are small) For every two different subsets $A,B \in
  \subsetfam$, $|A \cap B | \le \ell/2$.
\end{itemize}

For each $A \in \subsetfam$, we construct an instance $\combiband_A$ in
which the $i$-th arm has mean $\Delta$ if $i \in A$ and mean $0$
otherwise, where $\Delta$ is a small real number.  Clearly, a
$\delta$-correct algorithm must output the subset $A$ on instance
$\combiband_A$ with probability at least $1-\delta$. Intuitively
speaking, our lower bound works by showing that if an algorithm $\alg$
can correctly solve all $\combiband_A$'s, then there must exist two
different instances $\combiband_A$ and $\combiband_B$, such that $\alg$
can distinguish these two instance with a much smaller confidence of
$\delta/|\subsetfam|$. It is easy to see the later task requires
$\Omega(\Delta^{-2} \cdot \ln |\subsetfam|)$ samples by the change of
distribution method. Then, by a simple calculation, one can show that
$\gklow(\combiband_A) = \Theta(\Delta^{-2})$ for all $A \in \subsetfam$.

\paragraph{An $O(\gklow(\combiband_B) \cdot \operatorname{poly}(\ln
  n,\ln \delta^{-1}))$ Upper Bound.} To prove
Theorem~\ref{theo:worst-case-lowb-comb} (ii), for each instance of the
form $\combiband_B$ constructed above, we need to design a
$\delta$-correct algorithm $\alg_B$ which is particularly fast on that
instance.
To this end, we provide a surprisingly fast testing algorithm
$\alg_{\mathsf{help}}$ to distinguish between two hypotheses $x = 0$ and
$\|x\|_2 \ge 1$
with sample complexity
$\tilde O(n)$ (see Theorem~\ref{theo:ball-case} for the details).  This
is somewhat surprising, considering the following argument: for this
problem, uniform sampling seems to be a good method as the problem is
completely symmetric (no arm is more special than others).  If we sample
every arm once, we actually get a sample (which is an $n$-dimensional
vector) from a multivariate Gaussian $N(x, I_{n\times n})$.  To decide
whether the mean $x$ satisfies $x = 0$ or $\|x\|_2 \ge 1$ with
confidence 0.99, we need $O(n)$ samples from multi-variate Gaussian
$N(x, I_{n\times n})$, by a simple calculation.
This argument suggests that we need $O(n^2)$ arm pulls.  However, we
show that an interesting randomized sampling method only needs $\tilde
O(n)$ arm pulls.

\eat{
 And our algorithm simply works by outputting $B$ when $\alg_{\mathsf{help}}$ says $x = x_B$, and otherwise simulate an arbitrary $\delta$-correct algorithm for \combibandit. Clearly it takes the desired amount of samples on $\combiband_B$, and its correctness follows from the fact that when $\alg_{\mathsf{help}}$ operates correctly and says $x = x_B$, at least $x$ must satisfy $\|x - x_B\|_2 < r$, and therefore $B$ is also a correct answer for the given instance. \footnote{Actually the above algorithm may not have a desirable sample complexity bound in {\em expectation}, we have to use a simple trick called parallel simulation to transform it into the final algorithm we need, see Lemma~\ref{lem:parasim}. The same applies in the \generalbandit case.}
}
\paragraph{Worst Case Lower Bound for $\generalbandit$.} 
We consider the following special case of \generalbandit, which behaves like 
an $\mathsf{OR}$ function: 
namely, each arm has mean either $0$ or $\Delta$, and the goal is
to find out whether there is an arm with mean $\Delta$, where $\Delta$ is a small real number.

Let $\inst_{k}$ be an instance in which the $k$-th arm has mean $\Delta$, while other arms have mean $0$. On one hand, it is not hard to see that $\gklow(\inst_{k}) = \Theta(\Delta^{-2})$ via a simple calculation. On the other hand, we show that in order to solve all $\inst_{k}$'s correctly, an algorithm must in a sense find the $\Delta$-mean arm itself, and hence are forced to spend $\Omega(n \Delta^{-2})$ total samples in the worst-case. 
While the high level idea is simple, the formal proof is a bit technical
and is relegated to Appendix~\ref{sec:anotherlb}.


\subsection{Other Related Work}
An important and well-studied special case of \combibandit and \generalbandit is \bestarm, in which we would like identify the best single arm. For the PAC version of the problem,
\footnote{In the PAC version, our goal is to identify an $\varepsilon$-approximate optimal arm.
}
\cite{even2002pac} obtained an algorithm with sample 
complexity $O(n\varepsilon^{-2} \cdot \ln \delta^{-1})$,
which is also optimal in worse cases.
For the exact version of \bestarm,
a lower bound of $\Omega(\sum\nolimits_{i=2}^{n} \Gap{i}^{-2} \ln\delta^{-1})$, where $\Gap{i}$ denotes the gap between the $i$-th largest arm and the largest arm,
has been proved by~\cite{mannor2004sample}.
In a very early work, \cite{farrell1964asymptotic} established a worst-case
lower bound of
$\Omega(\Gap{2}^{-2}\ln\ln \Gap{2}^{-1})$
even if there are only two arms. 
An upper bound of 
$O(\sum\nolimits_{i=2}^{n} \Gap{i}^{-2} (\ln\ln\Gap{i}^{-1}+\ln\delta^{-1}))$ was achieved by the \textsf{Exponential-Gap Elimination} algorithm by~\cite{karnin2013almost},
matching Farrell's lower bound for two arms. Later,~\cite{jamieson2014lil} obtained the a more practical algorithm with the same theoretical sample complexity, based the confidence bounds derived from the 
law of iterative logarithm.
Very Recently, in \cite{chen2016towards,chen2016open}, 
the authors proposed an intriguing gap-entropy conjecture stating that the optimal instance-wise sample complexity of \bestarm\ is related to the entropy of a certain distribution of the arm gaps.  On a different line, \cite{garivier2016optimal} proposed an algorithm which is asymptotically optimal.
The high level ideas of our algorithms for \combibandit{} and \generalbandit{} are 
similar to the approach in \cite{garivier2016optimal}, which also computes the allocation
of samples using a mathematical program, and then take samples according to it.
In a high level, our algorithms are also similar to the
``explore-verify'' approach used in \cite{karnin2016verification}.

The natural generalization of \bestarm is \bestkarm{}, in which we would like to identify the top $k$ arms. The problem and its PAC versions have been also studied extensively in the past few years~\cite{kalyanakrishnan2010efficient,  gabillon2012best, gabillon2011multi, kalyanakrishnan2012pac, bubeck2013multiple, kaufmann2013information, zhou2014optimal, kaufmann2015complexity,chen2016pure,chen2017nearly}.

All aforementioned results are in the {\em fixed confidence} setting, where we need to output the correct answer with probability at least $1-\delta$, where $\delta$ is the given confidence level. 
Another popular setting is called the {\em fixed budget} setting, in which one aims to minimize the failure probability, subject to a fixed budget constraint on the total number of samples.
(see e.g., \cite{gabillon2012best,karnin2013almost,chen2014combinatorial,carpentier2016tight}).

Our problems are related to
the classic sequential hypothesis testing framework, which is 
pioneered by~\cite{wald1945sequential}.
In fact, they are closely related to the 
{\em active hypothesis testing} problem first studied by~\cite{chernoff1959sequential}.
In Wald's setting, the observations are predetermined, and we only need to design the stopping time.  In Chernoff's setting, the decision maker can 
choose different experiments to conduct, which result in different observations
about the underlying model. Chernoff focused on the case of binary hypotheses
and obtained an asymptotically optimal testing algorithm as
the error probability approaches to 0. 
Chernoff's seminal result has been extended to more than two hypothesis (see e.g., \cite{draglia1999multihypothesis,naghshvar2013active}).
In fact, the active multiple hypothesis testing is already quite general, and includes the bandit model as a special case.
Our work differs from the above line of work in the following aspects:
First, most of the work following Chernoff's approach (which extends Wald's approach) uses different variants of the SPRT (sequential probability ratio test). It is unclear how to compute such ratios (efficiently) for our combinatorial pure exploration problem. 
However, the computation problem is a major focus of our work
(we devote Section~\ref{sec:efficient} to discuss how to solve
the computation problem efficiently).
Second, the optimality of their results are in the asymptotic sense
(when $\delta\rightarrow 0$).
In the high dimension (i.e., $n$ is large) but moderate $\delta$ regime,
the additive term, which is independent of $\delta$ but dependent on $n$
(see Theorem~\ref{theo:worst-case-lowb-comb} and \ref{theo:lower-bound-general}), may become the dominate term.
Our work makes this term explicit and aims at minimizing it as well.



\section{Preliminaries}

Some naming conventions first. Typically, we use a lowercase letter to denote an {\em element}, e.g., an arm $a$ or an index $i$, uppercase letter to denote a set of elements, e.g., a set of arms $S$, and a letter in calligraphic font to denote a family of sets, e.g., a set of sets of arms $\mathcal{S}$.
Given a \combibandit{} instance $\combiband=(S, \subsetfam)$, $\mu_i$ denotes the mean of arm $i\in S$. For subset $A\subseteq S$, $\mu(A)$ denotes $\sum_{i\in A}\mu_i$.

For two sets $A, B$, we use $A \oplus B$ to denote the {\em symmetric difference} of $A$ and $B$. Namely, 
$$
A \oplus B = (A \backslash B ) \cup (B \backslash A).
$$
We need the following important lemma for calculating the confidence level of a subset of arms.

\begin{Lemma}
	\label{lem:sum_dev}
	Given a set of Gaussian arms $a_1,a_2,\dotsc, a_k$ with unit variance and means $\mu_1,\mu_2,\dotsc,\mu_k$, suppose we take $\tau_i$ samples in the $i^{th}$ arm, and let $X_i$ be its empirical mean. Then we have
	$$
	\Pr\left[\left| \sum_{i=1}^{k} X_i - \sum_i^k \mu_i \right| \ge \epsilon\right] \le 2\exp\left\{  - \frac{\epsilon^2}{2\sum_{i=1}^{k}1 / \tau_i } \right\}.
	$$
\end{Lemma}

\begin{proof}[Proof of Lemma~\ref{lem:sum_dev}]
  By assumption, $\sum_{i=1}^{k}X_i - \sum_{i=1}^{k}\mu_i$ follows the Gaussian distribution with mean $0$ and variance $\sum_{i=1}^{k}1/\tau_i$. The lemma hence follows from the tail bound of Gaussian distributions.
\end{proof}

\vspace{-0.2cm}
\paragraph{Tail bound of the $\chi^2$ distribution.} A $\chi^2$ distribution with $n$ degrees of freedom is the distribution of a sum of the squares of $n$ random variables drawn independently from the standard Gaussian distribution $\Normal(0, 1)$.
The following lemma, as a special case of \cite[Lemma 1]{laurent2000adaptive}, proves an exponential tail probability bound for $\chi^2$ distributions.

\begin{Lemma}\label{lem:chi2bound}
  Let $X$ be a $\chi^2$ random variable with $n$ degrees of freedom. For any $x > 0$, it holds that
    $$\pr{X\ge2n+3x}\le e^{-x}\text{.}$$
\end{Lemma}

Let $\KL(a_1,a_2)$ denote the Kullback-Leibler divergence from the distribution of arm $a_2$ to that of arm $a_1$. For two Gaussian arms $a_1$ and $a_2$ with means $\mu_1$ and $\mu_2$ respectively, it holds that
	$$\KL(a_1, a_2) = \frac{1}{2}(\mu_1-\mu_2)^2\text{.}$$
Moreover, let $$d(x, y) = x\ln(x/y) + (1-x)\ln[(1-x)/(1-y)]$$ denote the binary relative entropy function.

\paragraph{Change of Distribution.} The following ``Change of Distribution'' lemma, formulated by~\cite{kaufmann2015complexity}, characterizes the behavior of an algorithm when underlying distributions of the arms are slightly altered, and is thus useful for proving sample complexity lower bounds. 
Similar bounds are known in the sequential hypothesis testing literature (see e.g., \cite[p.283]{ghosh1970sequential})
In the following, $\Pr_{\alg,\combiband}$ and $\Ex_{\alg,\combiband}$ denote the probability and expectation when algorithm $\alg$ runs on instance $\combiband$.

\begin{Lemma}[Change of Distribution]\label{lem:CoD}
  Let $\alg$ be an algorithm that runs on $n$ arms, and let $\combiband
  = (a_1,a_2,\ldots,a_n)$ and $\combiband' = (a'_1,a'_2,\ldots,a'_n)$ be
  two sequences of $n$ arms. Let random variable $\tau_i$ denote the
  number of samples taken from the $i$-th arm. For any event $\event$ in
  $\mathcal{F}_\tau$, where $\tau$ is a stopping time with respect to
  the filtration $\{\mathcal{F}_t\}_{t\ge0}$, it holds that
  $$\sum_{i=1}^{n}\Ex_{\alg,\combiband}[\tau_i] \, \KL\left(a_i, a'_i\right)
  \ge
  d\left(\Pr_{\alg,\combiband}[\event],\Pr_{\alg,\combiband'}[\event]\right)\text{.}$$ 
\end{Lemma}


\section{Instance Lower Bound}\label{sec:instance-lowb}
	\subsection{Instance Lower Bound for \combibandit{}}\label{sec:comb-lowerb}
	Given a \combibandit{} instance $\combiband = (S, \subsetfam)$, let $O$ denote the optimal set in $\subsetfam$ (i.e., $O = \argmax_{A\in\subsetfam}\mu(A)$). We define $\gklow(\combiband)$ as the optimal value of the following mathematical program:
		\begin{equation}\begin{split}\label{eq:gklowdef}
			\textrm{minimize}~~&\sum_{i\in S}\tau_i\\
			\textrm{subject to}~~&\sum_{i\in O\oplus
                          A}1/\tau_i \le [\mu(O)-\mu(A)]^2 \qquad \forall A\in\subsetfam\\
			& \tau_i > 0,~\forall i\in S\text{.}
		\end{split}\end{equation}

	We prove Theorem~\ref{theo:LB}, which we restate in the following for convenience.

	\noindent\textbf{Theorem~\ref{theo:LB}} (restated)\textit{
		Let $\combiband = (S,\subsetfam)$ be an instance of $\combibandit$.
		For any $\delta \in (0,0.1)$
	  	and $\delta$-correct algorithm $\alg$ for $\combibandit$,
		$\alg$ takes
	    $\Omega(\gklow(\combiband) \ln \delta^{-1})$ 
		samples in expectation on $\combiband$.
	}

	\begin{proof}[Proof of Theorem \ref{theo:LB}]
		Fix $\delta\in(0,0.1)$, instance $\combiband$ and $\delta$-correct algorithm $\alg$. Let $n_i$ be the expected number of samples drawn from the $i$-th arm when $\alg$ runs on instance $\combiband$. Let $\alpha = d(1-\delta,\delta)/2$ and $\tau_i = n_i/\alpha$. It suffices to show that $\tau$ is a feasible solution for the program in \eqref{eq:gklowdef}, as it directly follows that
			$$\sum_{i=1}^{n}n_i
			=\alpha\sum_{i=1}^{n}\tau_i
			\ge\alpha\gklow(\combiband)
			=\Omega(\gklow(\combiband)\ln\delta^{-1})\text{.}$$
		Here the last step holds since for all $\delta\in(0, 0.1)$,
			$$d(1-\delta,\delta)
			=(1-2\delta)\ln\frac{1-\delta}{\delta}
			\ge0.8\ln\frac{1}{\sqrt{\delta}}
			=0.4\ln\delta^{-1}\text{.}$$

		To show that $\tau$ is a feasible solution, we fix $A\in\subsetfam$.
		Let $\Delta_i = c/n_i$, where
			$$c = \frac{2[\mu(O)-\mu(A)]}{\sum_{i\in O\oplus A}1/n_i}\text{.}$$
		We consider the following alternative instance $\combiband'$:
		the mean of each arm $i$ in $O\setminus A$ is decreased by $\Delta_i$,
		while the mean of each arm $i\in A\setminus O$ is increased by $\Delta_i$;
		the collection of feasible sets are identical to that in $\combiband$.
		Note that in $\combiband'$, the difference between the weights of $O$ and $A$ is given by
			$$\left(\mu(O)-\sum_{i\in O\setminus A}\Delta_i\right)-\left(\mu(A)+\sum_{i\in A\setminus O}\Delta_i\right)
			=\mu(O)-\mu(A)-c\sum_{i\in O\oplus A}1/n_i
			=-(\mu(O)-\mu(A)) < 0\text{.}$$
		In other words, $O$ is no longer optimal in $\combiband'$.

		Let $\event$ denote the event that algorithm $\alg$ returns $O$ as the optimal set. Note that since $\alg$ is $\delta$-correct, $\Pr_{\alg,\combiband}[\event]\ge1-\delta$ and $\Pr_{\alg,\combiband'}[\event]\le\delta$. Therefore, by Lemma \ref{lem:CoD},
			$$\sum_{i\in O\oplus A}n_i\cdot\frac12\Delta_i^2
			\ge d\left(\Pr_{\alg,\combiband}[\event],\Pr_{\alg,\combiband'}[\event]\right)\ge d(1-\delta,\delta)\text{.}$$

		Plugging in the values of $\Delta_i$'s yields
			$$2d(1-\delta,\delta)
			\le\sum_{i\in O\oplus A}n_i\cdot \frac{c^2}{n_i^2}
			=\frac{4[\mu(O)-\mu(A)]^2}{\left(\sum_{i\in O\oplus A}1/n_i\right)^2}\sum_{i\in O\oplus A}1/n_i
			=\frac{4[\mu(O)-\mu(A)]^2}{\sum_{i\in O\oplus A}1/n_i}\text{,}$$
		and it follows that
			$$\sum_{i\in O\oplus A}1/\tau_i\le[\mu(O)-\mu(A)]^2\text{.}$$
	\end{proof}
	\subsection{Comparison with Previous Lower Bound} 
	\label{sec:vschenlow}
	In this section, we show that our lower bound is no weaker than the lower bound in \cite{chen2014combinatorial}.
	Formally, we prove Theorem \ref{lem:vschenlow}. We restate it here for convenience. 
	
	\noindent\textbf{Theorem~\ref{lem:vschenlow}} (restated)
	\textit{
	Let $\combiband = (S,\subsetfam)$  be an instance of $\combibandit$,
$$
\gklow(\combiband)  \ge H_{C}(\combiband).
$$
	}

	\begin{proof}
	We start with the concept of {\em gap}, which was  defined as follows  in \cite{chen2014combinatorial}.
	Let $\combiband = (S,\subsetfam)$ be an instance of $\combibandit$ with the optimal set $O\in\subsetfam$.
	For each arm $i$, its gap $\Delta_i$ is defined as
	$$
\Delta_i = \begin{cases}
\mu(O) - \max_{O'\in\subsetfam, i\notin O'} \mu(O'), & i \in O,\\
\mu(O) - \max_{O'\in\subsetfam, i\in O'} \mu(O'), & i \notin O\text{.}
\end{cases}
	$$
	The {\em hardness} of the instance $\combiband$, $\chenlow(\combiband) $, is defined as
	$\chenlow(\combiband) = \sum_{i\in S} m_i',$
	where $m_i' = \Delta_i^{-2}$.
	
Consider the following mathematical program, which is essentially the same as Program \eqref{eq:gklowdef}, except for replacing
summation with maximization in the constraints. 
			\begin{equation}\begin{split}\label{eq:chenlow}
			\textrm{minimize}~~&\sum_{i\in S}\tau'_i\\
			\textrm{subject to}~~&\max_{i\in O\oplus
                          A}1/\tau'_i \le [\mu(O)-\mu(A)]^2 \qquad \forall A\in\subsetfam\\
			& \tau'_i > 0,~\forall i\in S\text{.}
		\end{split}\end{equation}
A simple observation is that, the optimal solution of Program \eqref{eq:chenlow} can be achieved by setting
$$
\tau'_i = \max_{i \in O \oplus A} [\mu(O) - \mu(A)]^{-2} = m'_i.
$$
Furthermore, every feasible solution of 
Program \eqref{eq:gklowdef} is also a feasible solution of  
Program \eqref{eq:chenlow}.
Theorem~\ref{lem:vschenlow} hence holds.
\end{proof}

	\subsection{Instance Lower Bound for \generalbandit{}}\label{sec:general-lowerb}
		For a \generalbandit{} instance
			$\inst = (\armseq, \anssetcol)$
		with mean profile $\mean\in\ansset$,
		we define $\gklow(\inst)$ as the optimal value
		of the following linear program:
			\begin{equation}\label{eq:genlbdef}\begin{split}
				\textrm{minimize}~~~~&\sum_{i=1}^{n}\tau_i\\
				\textrm{subject to}~~~~&\sum_{i=1}^{n}\left(\tilmean_i - \mean_i\right)^2 \tau_i \ge1,~\forall \tilmean\in\alt(\ansset),\\
				& \tau_i\ge 0\text{.}
			\end{split}\end{equation}
		Here $\alt(\ansset) = \bigcup_{\ansset'\in\anssetcol\setminus\{\ansset\}}\ansset'$.

		We prove that $\Omega(\gklow(\inst)\ln\delta^{-1})$ is an instance-wise sample complexity lower bound for instance $\inst$.
		The proof is very similar to that in \cite{garivier2016optimal}, and we provide it here for completeness.

		\begin{theo}\label{theo:genlb}
			Suppose $\delta\in(0,0.1)$ and $\inst$ is an instance
			of \generalbandit.
			Any $\delta$-correct algorithm for \generalbandit{} takes
				$\Omega(\gklow(\inst)\ln\delta^{-1})$
			samples in expectation on $\inst$.
		\end{theo}

		\begin{proof}
			\, Fix $\delta\in(0,0.1)$, instance $\inst$ and a $\delta$-correct algorithm $\alg$. Let $n_i$ be the expected number of samples drawn from the $i$-th arm when $\alg$ runs on instance $\inst$. Let $\alpha = 2d(1-\delta,\delta)$ and $\tau_i = n_i/\alpha$. It suffices to show that $\tau$ is a feasible solution for the program in \eqref{eq:genlbdef}, as it directly follows that the expected sample complexity of $\alg$ is lower bounded by
				$$\sum_{i=1}^{n}n_i
				=\alpha\sum_{i=1}^{n}\tau_i
				\ge\alpha\gklow(\inst)
				=\Omega(\gklow(\inst)\ln\delta^{-1})\text{.}$$
			Here the last step holds since for all $\delta\in(0, 0.1)$,
				$$d(1-\delta,\delta)
				=(1-2\delta)\ln\frac{1-\delta}{\delta}
				\ge0.8\ln\frac{1}{\sqrt{\delta}}
				=0.4\ln\delta^{-1}\text{.}$$

			To show that $\tau$ is a feasible solution, we fix $\tilmean\in\alt(\ansset)$. Let $\inst'$ be an alternative instance obtained by changing the mean profile in $\inst$ from $\mean$ to $\tilmean$.
			Let $\event$ denote the event that algorithm $\alg$ returns $\ansset$ as the optimal set. The $\delta$-correctness of $\alg$ guarantees that $\Pr_{\alg,\inst}[\event]\ge1-\delta$ and $\Pr_{\alg,\inst'}[\event]\le\delta$. By Lemma \ref{lem:CoD},
				$$\sum_{i=1}^{n}n_i\cdot\KL(\Normal(\mean_i, 1), \Normal(\tilmean_i, 1))
				\ge d\left(\Pr_{\alg,\inst}[\event],\Pr_{\alg,\inst'}[\event]\right)\ge d(1-\delta,\delta)\text{.}$$

			Plugging in $\alpha = 2d(1-\delta,\delta)$ and $\KL(\Normal(\mean_i, 1), \Normal(\tilmean_i, 1)) = \frac{1}{2}(\tilmean_i-\mean_i)^2$ yields
				$$\alpha = 2d(1-\delta,\delta)
				\le\sum_{i=1}^{n}n_i\cdot (\tilmean_i-\mean_i)^2
				=\alpha\sum_{i=1}^{n}(\tilmean_i-\mean_i)^2\tau_i\text{,}$$
			and it follows that
				$\sum_{i=1}^{n}(\tilmean_i-\mean_i)^2\tau_i\ge 1\text{.}$
		\end{proof}

\newcommand{\meanhat}[2]{\hat\mu^{(#1)}_{#2}}
\newcommand{\numsamp}[2]{m^{(#1)}_{#2}}
\newcommand{\opt}{\mathrm{opt}}
\newcommand{\maxoracle}{\mathsf{OPT}}
\newcommand{\goodevent}{\event^\mathrm{good}}
\newcommand{\badevent}{\event^\mathrm{bad}}
\newcommand{\Sample}{\textsf{Sample}}

\newcommand{\verify}{\textsf{Verify}}
\newcommand{\algeffi}{\textsf{EfficientGapElim}}
\newcommand{\algnaive}{\textsf{NaiveGapElim}}
\newcommand{\algpar}{\textsf{ParallelGapElim}}
\newcommand{\unique}{\textsf{Unique}}

\section{Optimal Algorithm for Combinatorial Bandit}
\label{sec:inefficient}

	In this section, we present an algorithm for \combibandit{} that nearly achieves the optimal sample complexity.
	We postpone the computationally efficient implementation of the
        algorithm to the next section. 

	\subsection{Overview}
		Our algorithm is based on a process of successive
                elimination. However, unlike the previous approaches~\cite{karnin2013almost,chen2015optimal,chen2016pure,chen2016towards,gabillon2016improved} which maintained a set of arms, our algorithm maintains a collection of candidate sets and performs the eliminations on them directly.
				Specifically, at the $r$-th round of the algorithm, we adopt a precision level $\epsilon_r\coloneqq 2^{-r}$, and maintain a set of candidates sets $\subsetfam_r \subseteq \subsetfam$; and the goal of the $r$-th round is to eliminate those sets in $\subsetfam_r$ with a optimality gap\footnote{The optimality gap for a set $A 
			\in \subsetfam$ is simply $\mu(O) - \mu(A)$ where $O$ is the optimal set.} of at least $\Theta(\epsilon_r)$.

A crucial difficulty in implementing the above elimination is that
it seems we have to compute the optimal set itself in order to
approximate the optimality gaps. To circumvent this problem, we instead approximate the gaps between \emph{every pair} of sets in $\subsetfam_r$, which certainly include the optimality gaps.
				Roughly speaking, we solve an optimization similar to that in \eqref{eq:gklowdef},
		and sample the arms accordingly to obtain an $O(\epsilon_r)$ approximation of
		the gap between \emph{every pair} of sets that are still under consideration.
		Now if there is a set pair $(A, B)$ such that the empirical mean of $B$ exceeds that of $A$ by $\epsilon_r$,
		we are certain that $A$ is not the optimal set,
		and thus we stop considering $A$ as a candidate answer.
		This process is repeated until only one set remains.

	\subsection{Algorithm}
		We first define a few useful subroutines that play important roles in the algorithm. Procedure $\simest$ takes as its input a set $\paraset \subseteq \subsetfam$ together with an accuracy parameter $\epsilon$ and a confidence level $\delta$. It outputs a vector $\{m_i\}_{i\in S}$ over $S$ that specifies the number of samples that should be taken from each arm, so that the difference between any two sets in $\paraset$ can be estimated to an accuracy of $\epsilon$ with probability $1 - \delta$.
	
		\begin{algorithm2e}[H]
			\caption{$\simest(\paraset, \epsilon, \delta)$}
			\KwIn{$\paraset$, accuracy parameter $\epsilon$, and confidence level $\delta$.}
			\KwOut{A vector $m$, indicating the number of samples to be taken from each arm.}
			Let $\{m_i\}_{i \in S}$ be the optimal solution of the following program:
			\begin{equation*}\begin{split}
				\textrm{minimize}~~&\sum_{i \in S}m_i\\
				\textrm{subject to}~~&\sum_{i \in A\oplus B}\frac{1}{m_i} \le \frac{\epsilon^2}{2\ln(2/\delta)},~\forall A, B\in \paraset\\
									 &m_i > 0,~\forall i \in S
			\end{split}\end{equation*}

			\textbf{return} $m$\;
		\end{algorithm2e}

		Procedure $\verify$ takes a sequence $\subsetfam_1,
                \ldots, \subsetfam_r$ of subsets of $\subsetfam$
                together with a confidence parameter $\delta$. Similar
                to $\simest$, it returns a vector $\{m_i\}_{i\in S}$ of
                the number of samples from each arm, so that the gap
                between the conjectured answer $\widehat O$, the only
                set in $\subsetfam_r$, and each set in $\subsetfam_k$
                can be estimated to $O(\epsilon_k)$ accuracy. (Recall
                that $\epsilon_k=2^{-k}$.) Notice
                that in general, the solutions returned by $\simest$ and
                $\verify$ are real-valued. We can simply get an
                integer-valued solution by rounding up without affecting the asymptotical performance.

		\begin{algorithm2e}[H]
			\caption{$\verify(\{\subsetfam_k\}, \delta)$}
			\KwIn{$\{\subsetfam_k\}_{k\in[r]}$ and confidence level $\delta$. It is guaranteed that $|\subsetfam_r|=1$.}
			\KwOut{A vector $m$, indicating the number of samples to be taken from each arm.}
			$\hat O \leftarrow$ the only set in $\subsetfam_r$;
			$\lambda \leftarrow 10$\;
			Let $\{m_i\}_{i \in S}$ be the optimal solution of the following program:
			\begin{equation*}\begin{split}
				\textrm{minimize}~~&\sum_{i \in S}m_i\\
				\textrm{subject to}~~&\sum_{i \in \hat O\oplus A}\frac{1}{m_i} \le \frac{(\epsilon_k/\lambda)^2}{2\ln(2/\delta)},~\forall k\in[r], A\in \subsetfam_k\\
									 &m_i > 0,~\forall i \in S
			\end{split}\end{equation*}

			\textbf{return} $m$\;
		\end{algorithm2e}
		 
		$\Sample$ is a straightforward sampling procedure: it
                takes a vector $m$, samples each arm $i\in S$ exactly
                $m_i$ times, and returns the empirical means. 
        We finish the description of all subroutines.
        
        \begin{rem}
        		In this section, we focus on the sample complexity and do not worry too much about the computation complexity of our algorithm.
        		It would be convenient to think that the subsets in 
        		$\subsetfam$ are given explicitly as input.
        		In this case, the convex programs in the subroutines
        		can be solved in time polynomial in $n$ and $|\subsetfam|$.
        		We will consider computational complexity issues when 
        		$\subsetfam$ is given implicitly in Section~\ref{sec:efficient}.
        \end{rem}

		Now, we describe our algorithm $\algnaive$ for \combibandit.
		$\algnaive$ proceeds in rounds. In each round $r$, it calls the $\simest$ procedure and samples the arms in $S$ accordingly, and then removes the sets with $\Theta(\epsilon_r)$ gaps from $\subsetfam_r$. When exactly one set $\hat O$ remains in $\subsetfam_r$ (i.e., the condition at line \ref{line:if} is met), the algorithm calls $\verify$ and $\Sample$, and verifies that the conjectured answer $\hat O$ is indeed optimal.

		\begin{algorithm2e}[H]
		\caption{$\algnaive(\combiband, \delta)$}
		\KwIn{\combibandit{} instance $\combiband = (S, \subsetfam)$ and confidence level $\delta$.}
		\KwOut{The answer.}
		$\subsetfam_1\leftarrow \subsetfam$;
		$\delta_0 \leftarrow 0.01$;
		$\lambda \leftarrow 10$\;
		\For{\upshape $r = 1 \text{ to } \infty$} {
			\If{$|\subsetfam_r| = 1$\label{line:if}} {
				$m \leftarrow \verify(\{\subsetfam_k\}_{k=1}^{r}, \delta/(r|\subsetfam|))$\;
				$\hat\mu \leftarrow \Sample(m)$\;\label{line:hatmu}
				$\hat O \leftarrow$ the only set in $\subsetfam_r$\;\label{line:hatO}
				\If{\upshape $\hat\mu(\hat O) - \hat\mu(A) \ge \epsilon_k/\lambda$ for all  $A\in\subsetfam\setminus\subsetfam_k$
				and all $k\in[r]$	
					\label{line:test}} {
					\textbf{return} $\hat O$\;
				} \Else {
					\textbf{return} error\;
				}
			}
			$\epsilon_r \leftarrow 2^{-r}$;
			$\delta_r \leftarrow \delta_0/(10r^2|\subsetfam|^2)$\;
			$\numsamp{r}{} \leftarrow \simest(\subsetfam_r, \epsilon_r / \lambda, \delta_r)$\;
			$\meanhat{r}{} \leftarrow \Sample(\numsamp{r}{})$\;
			$\opt_r \leftarrow \max_{A\in\subsetfam_r}\meanhat{r}{}(A)$\;
			$\subsetfam_{r+1} \leftarrow \{A \in \subsetfam_r: \meanhat{r}{}(A) \ge \opt_r - \epsilon_r/2 - 2\epsilon_r/\lambda \}$\;
		}
		\end{algorithm2e}

	\subsection{Correctness of \algnaive}
\label{sec:correctness-algnaive}

		We formally state the correctness guarantee of $\algnaive$ in the following lemma.

		\begin{Lemma}\label{lem:naivecorrect}
                  For any $\delta\in(0, 0.01)$ and \combibandit{}
                  instance $\combiband$, $\algnaive(\combiband, \delta)$
                  returns the correct answer with probability $1 - \delta_0 - \delta$, and returns an incorrect answer w.p.\ at most~$\delta$.
		\end{Lemma}

                The proof proceeds as follows. We first define two
                ``good events'' $\goodevent_0$ and $\goodevent$, which
                happen with probability at least $1-\delta_0$ and
                $1-\delta$, respectively. Our algorithm always returns
                the correct answer conditioned on
                $\goodevent_0\cap\goodevent$, and it either returns the
                correct answer or reports an error conditioned on
                $\goodevent$. This implies that \algnaive{} is
                $(\delta+\delta_0)$-correct. This is not ideal since
                $\delta_0$ is a fixed constant, so in
                Section~\ref{sec:parallel} we use a parallel simulation
                construction to boost its success probability to $1 -
                \delta$, while retaining the same sample complexity.

		\paragraph{Good events.} Define $\goodevent_{0, r}$ as the event that either the algorithm terminates before round $r$, or for all $A, B\in\subsetfam_r$, 
			$$\left|(\meanhat{r}{}(A) - \meanhat{r}{}(B)) - (\mu(A)-\mu(B))\right| < \epsilon_r / \lambda\text{.}$$
		Here $\lambda$ is the constant in \algnaive{}. Moreover, we define $\goodevent_0$ as the intersection of $\{\goodevent_{0,r}\}$, i.e.,
			$$\goodevent_0\coloneqq\bigcap_{r=1}^{\infty}\goodevent_{0,r}\text{.}$$
		$\goodevent$ is defined as the event that for all $A\in\subsetfam_r$,
			$$\left|(\hat\mu(\hat O) - \hat\mu(A)) - (\mu(\hat O)-\mu(A))\right| < \epsilon_r / \lambda\text{.}$$
		Here $\hat\mu$ and $\hat O$ are defined at lines \ref{line:hatmu} and \ref{line:hatO} in \algnaive{}.

                \begin{Lemma}
                  \label{lem:good-event-prob}
                  $\pr{\goodevent_0} \geq 1-\delta_0$ and $\pr{\goodevent} \geq 1-\delta$.
                \end{Lemma}
                \begin{proof}[Proof of Lemma~\ref{lem:good-event-prob}]
                  Since $\numsamp{r}{}$ is a feasible solution of the
                  program in $\simest(\subsetfam_r, \epsilon_r/\lambda,
                  \delta_r)$, it holds for all $A,B\in\subsetfam_r$ that
			$$\sum_{i\in A\oplus B}\frac{1}{\numsamp{r}{i}}\le\frac{(\epsilon_r/\lambda)^2}{2\ln(2/\delta_r)}\text{.}$$
                        By Lemma \ref{lem:sum_dev},
                        \begin{equation*}\begin{split}
                            &\pr{\left|(\meanhat{r}{}(A)-\meanhat{r}{}(B))-(\mu(A)-\mu(B))\right|\ge\epsilon_r/\lambda}\\
                            &= \pr{\left|(\meanhat{r}{}(A\setminus B)-\meanhat{r}{}(B\setminus A))-(\mu(A\setminus B)-\mu(B\setminus A))\right|\ge\epsilon_r/\lambda}\\
                            &\leq 2\exp\left\{-\frac{(\epsilon_r/\lambda)^2}{2\sum_{i\in A\oplus B}1/\numsamp{r}{i}}\right\}\\
                            &\leq 2\exp\left(-\ln(2/\delta_r)\right) =
                            \delta_r\text{.}
                          \end{split}\end{equation*}
                        By a union bound over all possible $A,
                        B\in\subsetfam_r$, we have
                        $\pr{\overline{\goodevent_{0,r}}}\le
                        |\subsetfam|^2\,\delta_r =
                        \delta_0/(10r^2)\text{.}$ It follows from
                        another union bound that
                        \begin{equation*}
                          \pr{\goodevent_0}	\ge 1 - \sum_{r=1}^{\infty}\pr{\overline{\goodevent_{0,r}}}
                          \ge 1 - \sum_{r=1}^{\infty}\frac{\delta_0}{10r^2}
                          \ge 1 - \delta_0\text{.}
                        \end{equation*}
                        A similar union bound argument over all
                        $k\in[r]$ and $A\in\subsetfam_r$ yields that
                        $\pr{\goodevent} \ge 1 - \delta\text{.}$
                      \end{proof}

                      \paragraph{Implications of good events.} Let $O$ denote the optimal set in $\subsetfam$, i.e., $O = \argmax_{A\in\subsetfam}\mu(A)$. Throughout the analysis of our algorithm, it is useful to group the sets in $\subsetfam$ based on the gaps between their weights and $\mu(O)$. Formally, we define $G_r$ as
                \begin{gather}
                  G_r \coloneqq
                  \{A\in\subsetfam:\mu(O)-\mu(A)\in(\epsilon_{r+1},\epsilon_r]\}\text{.}
                \end{gather}
		We also adopt the shorthand notation
			$G_{\ge r} = \{O\}\cup\bigcup_{k = r}^{\infty}G_k$.

		\begin{Lemma}\label{lem:survive}
			Conditioning on $\goodevent_0$, $O \in \subsetfam_r$ for all $r \ge 1$.
		\end{Lemma}
		\begin{proof}[Proof of Lemma \ref{lem:survive}]
			Suppose for a contradiction that $O \in \subsetfam_r \setminus \subsetfam_{r+1}$ for some $r$. By definition of $\subsetfam_{r+1}$, it holds that
				$$\meanhat{r}{}(O) < \opt_r - \epsilon_r/2 - 2\epsilon_r/\lambda\text{.}$$
			Observe that $\opt_r = \meanhat{r}{}(A)$ for some $A\in\subsetfam_r$.
			Therefore,
				$\meanhat{r}{}(O) - \meanhat{r}{}(A) < -(1/2+2/\lambda)\epsilon_r\text{.}$
			Since $O$ is the maximum-weight set with respect to $\mu$,
				$\mu(O) - \mu(A) \ge 0\text{.}$
			These two inequalities imply that
				$$\left|(\meanhat{r}{}(O) - \meanhat{r}{}(A))-(\mu(O) - \mu(A))\right|>0 - \left[-(1/2+2/\lambda)\epsilon_r\right] > \epsilon_r/\lambda\text{,}$$
			which happens with probability zero conditioning on event $\goodevent_0$, since $O, A\in\subsetfam_r$.
		\end{proof}

		The following lemma, as a generalization of Lemma \ref{lem:survive} to all sets in $\subsetfam$, states that the sequence $\{\subsetfam_r\}$ is an approximation of $\{G_{\ge r}\}$ conditioning on event $\goodevent_0$.
		\begin{Lemma}\label{lem:chain}
			Conditioning on $\goodevent_0$, $G_{\ge r}\supseteq \subsetfam_{r+1}\supseteq G_{\ge r + 1}$ for all $r \ge 1$.
		\end{Lemma}
		\begin{proof}[Proof of Lemma \ref{lem:chain}]
			We first prove the left inclusion. Suppose that $A\in \subsetfam_r$ and $A\notin G_{\ge r}$. By definition of $G_{\ge r}$, $\mu(O) - \mu(A) > \epsilon_r$. Conditioning on $\goodevent_0$, we have $O\in\subsetfam_r$ by Lemma \ref{lem:survive}, and thus
				$$\meanhat{r}{}(O)-\meanhat{r}{}(A)
				> \mu(O) - \mu(A) - \epsilon_r / \lambda
				> (1 - 1 / \lambda) \epsilon_r\text{.}$$
			Recall that $\opt_r = \max_{A\in\subsetfam_r}\meanhat{r}{}(A) \ge \meanhat{r}{}(O)$, and $1 - 1/\lambda > 1/2 + 2/\lambda$ by our choice of $\lambda$. It follows that
				$$\meanhat{r}{}(A)
				< \meanhat{r}{}(O) - (1 - 1 / \lambda) \epsilon_r
				< \opt_r - \epsilon_r / 2 - 2\epsilon_r / \lambda\text{,}$$
			and thus $A\notin\subsetfam_{r+1}$.

			Then we show that $A\in G_{\ge r + 1}$ implies $A\in \subsetfam_{r+1}$. Note that $A\in G_{\ge r + 1}$ implies $\mu(O) - \mu(A) \le \epsilon_{r+1} = \epsilon_r / 2$, and thus,
				$$\meanhat{r}{}(O) - \meanhat{r}{}(A)
				\le \mu(O) - \mu(A) + \epsilon_r / \lambda
				\le (1/2 + 1/\lambda)\epsilon_r\text{.}$$

			Moreover, since $\opt_r = \meanhat{r}{}(B)$ for some $B\in\subsetfam_r$, it holds that
				$$\opt_r - \meanhat{r}{}(O)
				= \meanhat{r}{}(B) - \meanhat{r}{}(O)
				\le \mu(B) - \mu(O) + \epsilon_r / \lambda
				\le \epsilon_r / \lambda\text{.}$$
			Adding the two inequalities above yields
				$$\meanhat{r}{}(A) \ge \opt_r - (1/2 + 2/\lambda)\epsilon_r\text{.}$$
			By definition of $\subsetfam_{r+1}$ in \algnaive{}, $A\in\subsetfam_{r+1}$, which completes the proof.
		\end{proof}

		\paragraph{Correctness conditioning on $\goodevent_0\cap\goodevent$.} By Lemma \ref{lem:survive}, conditioning on $\goodevent_0$, the correct answer is in $\subsetfam_r$ for every $r$. This guarantees that whenever the algorithm enters the if-statement (i.e., when $|\subsetfam_r| = 1$), it holds that $\subsetfam_r=\{O\}$. Moreover, let $r^*$ be a sufficiently large integer such that $G_{\ge r^*} = G_{\ge r^* + 1} = \{O\}$. Then Lemma \ref{lem:chain} implies that $\subsetfam_{r^* + 1} = \{O\}$, and consequently the algorithm eventually enters the if-statement, either before or at round $r^* + 1$.

		Now we show that the algorithm always returns the correct answer $O$ instead of reporting an error, conditioning on $\goodevent_0\cap\goodevent$. Fix $A \in \subsetfam\setminus\{O\}$. Let $r$ be the largest integer such that $A\in\subsetfam_r$, i.e., $A\in\subsetfam_r\setminus\subsetfam_{r+1}$. By Lemma \ref{lem:chain}, we have $\subsetfam_{r+1}\supseteq G_{\ge r+1}$. It follows that $A\notin G_{\ge r+1}$, and thus $\mu(O)-\mu(A) > \epsilon_{r+1} = \epsilon_r/2$.

		Recall that since $O, A\in\subsetfam_r$, conditioning on event $\goodevent$, $\hat\mu(O)-\hat\mu(A)$ is within an additive error of $\epsilon_r/\lambda$ to $\mu(O)-\mu(A)$. Therefore,
			$$\hat\mu(O)-\hat\mu(A) > \mu(O)-\mu(A) - \epsilon_r/\lambda > (1/2 - 1/\lambda)\epsilon_r > \epsilon_r/\lambda\text{.}$$
		Here the last step follows from our choice of parameter $\lambda$.

		Consequently, the condition at line \ref{line:test} of \algnaive{} is always met conditioning on $\goodevent_0\cap\goodevent$, and \algnaive{} returns the correct answer $O$.

		\paragraph{Soundness conditioning on $\goodevent$.} Finally, we show that conditioning on $\goodevent$, \algnaive{} either returns the correct answer or reports an error. Suppose that when the algorithm enters the if-statement at line \ref{line:if}, $\subsetfam_r$ is equal to $\{\hat O\}$ for some $\hat O\ne O$. Let $r$ be the unique integer that satisfies $O\in\subsetfam_r\setminus\subsetfam_{r+1}$. Recall that since $O, \hat O\in\subsetfam_r$, conditioning on event $\goodevent$,
			$$\left|(\hat\mu(\hat O) - \hat\mu(O)) - (\mu(\hat O) - \mu(O))\right| < \epsilon_r/\lambda\text{.}$$
		By definition, $\mu(\hat O) - \mu(O) < 0$, and it follows that
			$$\hat\mu(\hat O) - \hat\mu(O)) < \mu(\hat O) - \mu(O) + \epsilon_r/\lambda < \epsilon_r/\lambda\text{.}$$
		This guarantees that the condition at line \ref{line:test} is not met when $\hat O\ne O$, and thus the algorithm does not incorrectly return $\hat O$.

	\subsection{Sample Complexity}\label{sec:sample-algnaive}
		We analyze the sample complexity of the \algnaive{} algorithm conditioning on event $\goodevent_0\cap\goodevent$. Let $\Delta = \mu(O) - \max_{A\in\subsetfam\setminus\{O\}}\mu(A)$ denote the gap between the set with the second largest weight in $\subsetfam$ and the weight of $O$.

		\begin{Lemma}\label{lem:naivesample}
			For any $\delta\in(0,0.01)$ and \combibandit{} instance $\combiband$,
			$\algnaive(\combiband, \delta)$ takes
				$$O\left(\gklow(\combiband)\ln\delta^{-1} + \gklow(\combiband)\ln\Delta^{-1}\left(\ln\ln\Delta^{-1} + \ln|\subsetfam|\right)\right)$$
			samples conditioning on event $\goodevent_0\cap\goodevent$.
		\end{Lemma}

		\begin{proof}
			Recall that for a \combibandit{} instance $\combiband = (S, \subsetfam)$, $\gklow(\combiband)$ is defined as
				$$\gklow(\combiband) \coloneqq \sum_{i\in S}\tau^*_i\text{,}$$
			where $\tau^*$ denotes the optimal solution to the following program:
			\begin{equation}\begin{split}
				\textrm{minimize}~~&\sum_{i\in S}\tau_i\\
				\textrm{subject to}~~&\sum_{i\in O\oplus A}1/\tau_i \le [\mu(O)-\mu(A)]^2,~\forall A\in\subsetfam\\
				& \tau_i > 0,~\forall i\in S\text{.}
			\end{split}\end{equation}

			For each $r$, we construct a feasible solution to the corresponding program in $\simest(\subsetfam_r, \epsilon_r/\lambda, \delta_r)$, thereby proving an upper bound on the number of samples taken in round $r$. Let $\alpha = 16\lambda^2\ln(2/\delta_r)$ and $m_i = \alpha\tau^*_i$. Fix $A, B\in\subsetfam_r$. By Lemma \ref{lem:chain}, we have $A, B \in G_{\ge r - 1}$, and thus $\mu(O)-\mu(A)\le\epsilon_{r-1}$ and $\mu(O)-\mu(B)\le\epsilon_{r-1}$. Therefore,
				\begin{equation*}\begin{split}
					\sum_{i\in A\oplus B}1/m_i
					&\leq \alpha^{-1}\left(\sum_{i\in O\oplus A}1/\tau^*_i+\sum_{i\in O\oplus B}1/\tau^*_i\right)\\
					&\le \alpha^{-1}\left[[\mu(O)-\mu(A)]^2+[\mu(O)-\mu(B)]^2\right]\\
					&\le 2\alpha^{-1}\epsilon_{r-1}^2 = \frac{(\epsilon_r/\lambda)^2}{2\ln(2/\delta_r)}\text{.}
				\end{split}\end{equation*}
			Here the second step holds since $\tau^*$ is a feasible solution to the program in \eqref{eq:gklowdef}. The third step follows from $\mu(O) - \mu(A) \le \epsilon_{r-1}$ and $\mu(O) - \mu(B) \le \epsilon_{r-1}$. Finally, the last step applies $\alpha = 16\lambda^2\ln(2/\delta_r)$.

			Therefore, $\{m_i\}$ is a valid solution to the program in \simest{}, and then the number of samples taken in round $r$ is upper bounded by
				$$\sum_{i\in S}m_i
				= \alpha\sum_{i\in S}\tau^*_i
				= O(\gklow(\combiband)\ln\delta_r^{-1})
				= O\left(\gklow(\combiband)\left(\ln r +\ln|\subsetfam|\right)\right)\text{.}$$
			The last step holds due to
				$$\ln\delta_r^{-1} = \ln(10r^2|\subsetfam|^2/\delta_0) = O(\ln r + \ln|\subsetfam|)\text{.}$$

			Recall that
				$\Delta = \mu(O) - \max_{A\in\subsetfam\setminus\{O\}}\mu(A)$.
			Let
				$r^* = \left\lfloor\log_2\Delta^{-1}\right\rfloor+1$
			be the smallest integer such that
				$\epsilon_{r^*} < \Delta$.
			As shown in the proof of correctness,
			the algorithm terminates before round $r^*+1$.
			Summing over all $r$ between $1$ and $r^*$ yields
				\begin{equation*}\begin{split}
					O\left(\gklow(\combiband)\sum_{r=1}^{r^*}\left(\ln r+\ln|\subsetfam|\right)\right)
				&=	O\left(r^*\cdot\gklow(\combiband)\left(\ln r^*+\ln|\subsetfam|\right)\right)\\
				&=	O\left(\ln\Delta^{-1}\cdot\gklow(\combiband)\left(\ln\ln\Delta^{-1}+\ln|\subsetfam|\right)\right)\text{.}
				\end{split}\end{equation*}

			It remains to upper bound the number of samples taken in the last round, denoted by round $r$. Let $\beta = 8\lambda^2\ln(2r|\subsetfam|/\delta)$, and $m_i = \beta \tau^*_i$. Fix $k\in[r]$ and $A\in\subsetfam_k$. By Lemma \ref{lem:chain}, we have $A\in G_{\ge k-1}$, which implies that $\mu(O)-\mu(A)\le\epsilon_{k-1}$. It also follows from Lemma \ref{lem:survive} that $\hat O = O$. Thus we have
				\begin{equation*}\begin{split}
					\sum_{i\in \hat O\oplus A}1/m_i
				&=	\beta^{-1}\sum_{i\in O\oplus A}1/\tau^*_i\\
				&\le\beta^{-1}[\mu(O)-\mu(A)]^2\\
				&\le4\beta^{-1}\epsilon_k^2 = \frac{(\epsilon_k/\lambda)^2}{2\ln(2r|\subsetfam|/\delta)}\text{.}
				\end{split}\end{equation*}
			In other words, $\{m_i\}$ is a feasible solution to the program in $\verify(\{\subsetfam_k\}, \delta/(r|\subsetfam|))$. Therefore, the number of samples taken in the last round $r$ is upper bounded by
				$$\sum_{i\in S}m_i = \beta\sum_{i\in S}\tau^*_i = O\left(\gklow(\combiband)\left(\ln\delta^{-1} + \ln\ln\Delta^{-1} + \ln|\subsetfam|\right)\right)\text{.}$$

			In sum, the number of samples taken by \algnaive{} conditioning on $\goodevent_0\cap\goodevent$ is
				$$O\left(\gklow(\combiband)\ln\delta^{-1} + \gklow(\combiband)\ln\Delta^{-1}\left(\ln\ln\Delta^{-1} + \ln|\subsetfam|\right)\right)\text{.}$$
		\end{proof}

	\subsection{Parallel Simulation}
\label{sec:parallel}

In the above sections, we showed that conditioning on the ``good''
events we had low sample complexity and returned correct answers. We now
show how to remove the conditioning and get a $\delta$-correct algorithm
with the same sample complexity in expectation (which is nearly
optimal), using a ``parallel simulation'' idea.
The idea was first used in the \bestarm\ problem in \cite{chen2015optimal}.

\begin{defi}\label{def:alg}
  An algorithm $\alg$ is $(\delta_0, \delta, A, B)$-correct if there
  exist two events $\event_0$ and $\event_1$ satisfying the following
  three conditions:
  \begin{enumerate}
  \item $\pr{\event_0}\ge1 - \delta_0 - \delta$ and $\pr{\event_1}\ge 1
    - \delta$.
  \item Conditioning on $\event_0$, $\alg$ returns the correct answer,
    and takes $O(A\ln\delta^{-1} + B)$ samples.
  \item Conditioning on $\event_1$, $\alg$ either returns the correct
    answer or terminates with an error.
  \end{enumerate}
\end{defi}

By Lemma \ref{lem:naivecorrect} and Lemma \ref{lem:naivesample},
\algnaive{} is a $(\delta_0, \delta, A, B)$-correct
algorithm for \combibandit{}, where $\event_0 = \goodevent_0 \cap
\goodevent$ and $\event_1 = \goodevent$, $\delta_0 = 0.01$, $A =
\gklow(\combiband)$ and
\[
	B = \gklow(\combiband)\ln\Delta^{-1}(\ln\ln\Delta^{-1} +
		\ln|\subsetfam|).
\]
The following lemma shows that we can obtain a $\delta$-correct
algorithm with the same $O(A\ln\delta^{-1} + B)$ sample complexity,
thus proving Theorem~\ref{theo:naive-upperb}.

		\begin{Lemma}[Parallel Simulation]\label{lem:parasim}
			If there is a $(\delta_0, \delta, A, B)$ algorithm for a sampling problem for $\delta_0 = 0.01$ and any $\delta < 0.01$, there is also a $\delta$-correct algorithm for any $\delta < 0.01$ that takes $O(A\ln\delta^{-1} + B)$ samples in expectation.
		\end{Lemma}

We postpone the proof of Lemma~\ref{lem:parasim} to Appendix~\ref{app:parasim}.


\newcommand{\checksol}{\mathsf{Check}}
\newcommand{\tilfam}{\widetilde\subsetfam}
\newcommand{\thetahi}{\theta^\mathrm{high}}
\newcommand{\thetalo}{\theta^\mathrm{low}}

\section{Optimal Algorithm for Combinatorial Bandit with Efficient Computation}
\label{sec:efficient}

In this section, we present a computationally efficient implementation
of the \algnaive{} algorithm. Recall that $\algnaive$ maintains a
sequence of set families $\{\subsetfam_r\}$, which contain the sets
still under consideration at the beginning of each round $r$. As
$|\subsetfam|$, the number of feasible sets, is typically exponential in
the number of arms, it may be computationally expensive to compute
$\{\subsetfam_r\}$ explicitly. The key to computational efficiency is to
find a compact representation of $\{\subsetfam_r\}$. In this paper, we
represent $\subsetfam_{r+1}$ using the empirical means $\meanhat{r}{}$
and some carefully chosen threshold $\theta_r$:
\[ \subsetfam_{r+1}=\{A\in\subsetfam:\meanhat{r}{}(A)\ge\theta_r\}. \]
Consequently, we have to adapt the procedures in \algnaive{}, including
$\simest$ and $\verify$, so that they work with this implicit
representation of set families.  To this end, we use the
$\varepsilon$-approximate Pareto curve framework of~\cite{papadimitriou2000approximability}. This technique
allows us to implement our subroutines in polynomial-time with respect
to the input size and $1/\epsilon$, if we relax the constraints in the
subroutines by an multiplicative factor of $1 + \epsilon$. In
particular, if $1/\epsilon$ is upper bounded by the sample complexity of
the instance, we would obtain a computationally efficient implementation
of the algorithm.

\subsection{Algorithm}

We give a simplified version of the algorithm, and then later boost its
probability of success by a parallel simulation (Lemma
\ref{lem:parasim}). The algorithm relies on computationally efficient
implementations of the subroutines $\simest$ and $\verify$, as well as
three new procedures $\unique$, $\checksol$ and $\maxoracle$. We start
by introducing the syntax and performance guarantees of these
procedures, and postpone their efficient implementation to
Section~\ref{sec:efficomp}.

Procedure $\simest$ takes as its input a weight $\mu$ on $S$, two
thresholds $\thetahi$ and $\thetalo$, together with an accuracy
parameter $\epsilon$ and a confidence level $\delta$, and outputs a
vector $\{m_i\}_{i\in S}$ indicating the number of samples to be taken
from each arm in $S$ to estimate the difference between any two sets in
$\{A\in\subsetfam \mid \mu(A)\ge\thetahi\}$ to an accuracy of $\epsilon$
with confidence $1 - \delta$.
This new procedure is akin to the version in
Section~\ref{sec:inefficient}, where we set $\paraset =
\{A'\in\subsetfam:\mu(A')\ge\thetahi\}$. While the lower threshold
$\thetalo$ is not explicitly used, it gives us the approximation
guarantee of the procedure: indeed, while $\simest$ will be guaranteed
to output a feasible solution to the original program, the resulting
objective will be a constant approximation of the \textit{tightened}
program obtained by replacing $\{A'\in\subsetfam:\mu(A')\ge\thetahi\}$
with $\{A'\in\subsetfam:\mu(A')\ge\thetalo\}$ in the constraints.  A
detailed specification of $\simest$ appears in
Section~\ref{com:spec}.
	
\begin{algorithm2e}[H]
  \caption{$\simest(\mu, \thetahi, \thetalo, \epsilon, \delta)$}
  \label{algo:simest}
  \KwIn{Mean vector $\mu$, thresholds $\thetahi$ and $\thetalo$, accuracy parameter $\epsilon$, confidence level $\delta$.}
  \KwOut{A vector $m$, indicating the number of samples to be taken from each arm.}
  Let $\{m_i\}_{i \in S}$ be an approximate solution to the following program:
  \begin{equation*}\begin{split}
      \textrm{minimize}~~&\sum_{i \in S}m_i\\
      \textrm{subject to}~~&\sum_{i \in A\oplus B}\frac{1}{m_i} \le \frac{\epsilon^2}{2\ln(2/\delta)},~\forall A, B\in \{A'\in\subsetfam:\mu(A')\ge\thetahi\}\\
      &m_i > 0,~\forall i \in S
    \end{split}\end{equation*}
  
  \textbf{return} $m$\;
\end{algorithm2e}

Similarly, procedure $\verify$ takes a sequence of means
$\{\meanhat{k}{}\}$, two threshold sequences $\{\thetahi_k\}$ and
$\{\thetalo_k\}$, together with a confidence parameter $\delta$. It
returns a vector $\{m_i\}$, indicating the number of samples from each
arm, so that the gap between the conjectured answer $\hat O$ and each
set in $\{A\in\subsetfam:\meanhat{k-1}{}(A)\ge\thetahi_{k-1}\}$ can be
estimated to $O(\epsilon_k)$ accuracy. As in $\simest$, the resulting
objective value is guaranteed to be bounded by a constant times the
optimal value of the tightened program, obtained by replacing
$\thetahi_{k-1}$ with $\thetalo_{k-1}$ in the contraint.

\begin{algorithm2e}[H]
  \caption{$\verify(\{\meanhat{k}{}\}, \{\thetahi_k\}, \{\thetalo_k\}, \delta)$}
  \KwIn{A sequence $\{\meanhat{k}{}\}_{k=0}^{r-1}$ of empirical means, threshold sequences $\{\thetahi_k\}_{k=0}^{r-1}$ and $\{\thetalo_k\}_{k=0}^{r-1}$, together with a confidence level $\delta$.}
  \KwOut{A vector $m$, indicating the number of samples to be taken from each arm.}
  $\lambda \leftarrow 10$;
  $\hat O \leftarrow \argmax_{A\in\subsetfam}\meanhat{r}{}(A)$\;
  Let $\{m_i\}_{i \in S}$ be an approximate solution to the following program:
  \begin{equation*}\begin{split}
      \textrm{minimize}~~&\sum_{i \in S}m_i\\
      \textrm{subject to}~~&\sum_{i \in \hat O\oplus A}\frac{1}{m_i} \le \frac{(\epsilon_k/\lambda)^2}{2\ln(2/\delta)},~\forall k\in[r], A\in\{A'\in\subsetfam:\meanhat{k-1}{}(A') \ge \thetahi_{k-1}\}\\
      &m_i > 0,~\forall i \in S
    \end{split}\end{equation*}
  
  \textbf{return} $m$\;
\end{algorithm2e}

The $\algeffi$ algorithm (Algorithm~\ref{alg:algeffi}) proceeds in
rounds. At round $r$, $\algeffi$ first calls the subroutine $\unique$ to
determine whether exactly one set survives (i.e., has a weight greater
than $\theta_{r-1}-\epsilon_{r-1}/\lambda$ with respect to
$\meanhat{r-1}{}$. If so, the algorithm invokes $\verify$, $\Sample$ (a
straightforward sampling procedure) and $\checksol$ (a procedure
analogous to Line~\ref{line:test} in \algnaive{}), in order to verify
that the conjectured answer $\hat O$ is indeed optimal. The algorithm
terminates and depending on these tests, returns either $\hat O$ or an
error.

Otherwise, $\algeffi$ calls $\simest$ and $\Sample$ to estimate the
means to sufficient accuracy. After that, $\maxoracle$ is called to
compute the approximately optimal set among the sets under
consideration. Finally, the algorithm computes the threshold for the
next round based on $\opt_r$.

\begin{algorithm2e}[ht]
  \caption{$\algeffi(\combiband, \delta)$}
  \label{alg:algeffi}
  \KwIn{\combibandit{} instance $\combiband = (S, \subsetfam)$ and confidence level $\delta$.}
  \KwOut{The answer.}
  $\meanhat{0}{}\leftarrow\vec{0}$;
  $\theta_0 \leftarrow 0$\;
  $\delta_0 \leftarrow 0.01$;
  $\lambda \leftarrow 20$\;
  \For{\upshape $r = 1 \text{ to } \infty$} {
    \If{$\unique(\meanhat{r-1}{}, \theta_{r-1}-\epsilon_{r-1}/\lambda)$\label{line:effiIf}} {
      $m \leftarrow \verify(
      \{\meanhat{k}{}\}_{k=0}^{r-1},
      \{\theta_k-\epsilon_k/\lambda\}_{k=0}^{r-1},
      \{\theta_k-2\epsilon_k/\lambda\}_{k=0}^{r-1},
      \delta/(r|\subsetfam|)
      )$\;
      $\hat\mu \leftarrow \Sample(m)$\label{line:effhatmu}\;
      $\hat O \leftarrow \argmax_{A\in\subsetfam}\meanhat{r-1}{}(A)$\label{line:effhatO}\;
      \If{\upshape$\checksol(\hat O, \meanhat{k}{}, \hat\mu, \theta_{k}, \epsilon_k/\lambda)$ for all $k\in[r-1]$\label{line:efficheck}} {
        \textbf{return} $\hat O$\;
      } \Else {
        \textbf{return} error\;
      }
    }
    $\epsilon_r \leftarrow 2^{-r}$;
    $\delta_r \leftarrow \delta_0/(10r^3|\subsetfam|^2)$\;
    $\numsamp{r}{} \leftarrow \sum_{k=1}^{r}\simest(
    \meanhat{k-1}{},
    \theta_{k-1}-\epsilon_{k-1}/\lambda,
    \theta_{k-1}-2\epsilon_{k-1}/\lambda,
    \epsilon_k / \lambda,
    \delta_r
    )$\;
    $\meanhat{r}{} \leftarrow \Sample(\numsamp{r}{})$\;
    $\opt_r \leftarrow \maxoracle(\meanhat{r-1}{}, \theta_{r-1}, \meanhat{r}{}, \epsilon_{r-1}/\lambda)$\;
    $\theta_r \leftarrow \opt_r - (1/2+2/\lambda)\epsilon_r$\;
  }
\end{algorithm2e}

\subsection{Specification}
\label{com:spec}

We formally state the performance guarantees of the subroutines in
\algeffi{}, which are crucial to the analysis of the algorithm. In
Section~\ref{sec:efficomp} we discuss implementations that meet these
specifications.

\begin{enumerate}
\item Given weights $\mu$ and threshold $\theta$, $\unique(\mu, \theta)$
  correctly decides whether there is exactly one set $A\in\subsetfam$
  such that $\mu(A)\ge\theta$.

\item Both $\simest$ and $\verify$ return \textit{feasible} solutions to
  the programs defined in the procedures. Moreover, the resulting
  objective function should be at most a constant times the optimal
  value of the tightened programs obtained by replacing $\thetahi$ with
  $\thetalo$ (or replacing $\{\thetahi_k\}$ with $\{\thetalo_k\}$) in
  the constraints.

\item Given empirical means $\mu$, threshold $\theta$, weight $w$, and
  accuracy level $\epsilon$, $\maxoracle(\mu, \theta, w, \epsilon)$
  returns a set $A \in \subsetfam$ such that: (a) $\mu(A) \ge
  \theta-\epsilon$ (i.e., $A$ is approximately feasible); (b) $w(A) \ge
  \max_{B\in\subsetfam, \mu(B)\ge\theta}w(B) - \epsilon$ (i.e., $A$ is
  approximately optimal).

\item When $\checksol(\hat O, \meanhat{k}{}, \hat\mu, \theta, \epsilon)$
  is called, and it holds that $\hat\mu(\hat O)-\hat\mu(A)\ge 2\epsilon$
  for all $A\in\subsetfam$ such that $\meanhat{k}{}(A)<\theta$, the
  procedure returns ``true''. If $\hat\mu(\hat O)-\hat\mu(A) \le
  \epsilon$ for \textit{some} $A\in\subsetfam$ such that
  $\meanhat{k}{}(A)<\theta - \epsilon$, the procedure always returns
  ``false''. In other cases, the procedure may return arbitrarily.
\end{enumerate}

\subsection{Analysis of \algeffi{}}

We state the performance guarantees of algorithm $\algeffi$ in the following two lemmas. The proofs are essentially identical to those in
Sections~\ref{sec:correctness-algnaive} and \ref{sec:sample-algnaive}, and are therefore postponed to Appendix~\ref{app:efficientmiss}. 

\begin{Lemma}\label{lem:efficorrect}
  For any $\delta\in(0, 0.01)$ and \combibandit{} instance $\combiband$,
  $\algeffi(\combiband, \delta)$ returns the correct answer with
  probability $1 - \delta_0 - \delta$, and returns an incorrect answer
  w.p.\ at most~$\delta$.
\end{Lemma}

Recall that $\Delta = \mu(O) - \max_{A\in\subsetfam\setminus\{O\}}\mu(A)$ is the gap between the set
with the second largest weight in $\subsetfam$ and the weight of $O$.

\begin{Lemma}\label{lem:effisample}
  For any $\delta\in(0,0.01)$ and \combibandit{} instance $\combiband$,
  $\algeffi(\combiband, \delta)$ takes
  $$O\left(\gklow(\combiband)\ln\delta^{-1} + \gklow(\combiband)\ln^2\Delta^{-1}\left(\ln\ln\Delta^{-1} + \ln|\subsetfam|\right)\right)$$
  samples conditioning on event $\goodevent_0\cap\goodevent$.
\end{Lemma}

Lemmas~\ref{lem:efficorrect}, \ref{lem:effisample} and \ref{lem:parasim} imply that there is a $\delta$-correct algorithm that matches the sample complexity stated in Theorem~\ref{theo:effi-comb}.
It remains to implement the subroutines specified in Section~\ref{com:spec} efficiently.


\subsection{Efficient Computation via $\varepsilon$-approximate Pareto Curve}
\label{sec:efficomp}

In this section, we propose a general framework for efficiently
implementing the subroutines specified in Section \ref{com:spec}, thus proving Theorem~\ref{theo:effi-comb}.  Here,
by ``efficient'', we mean the time complexity of the
algorithm is bounded by a function polynomial both in $n$ and the {\em sample
  complexity} of the algorithm. Indeed, for any natural algorithm, the
{\em time complexity} is at least the same as the {\em sample
  complexity}.  We use the concept of {\em $\varepsilon$-approximate
  Pareto curve}, a general framework for multi-objective optimization,
which was first introduced by~\cite{papadimitriou2000approximability}.

In this section, we only need {\em bi-objective optimization problems}, i.e., problems with two objective functions. 
For a bi-objective optimization problem, for each instance $x$, we denote $F(x)$ to be its feasible solution space.
For each feasible solution $s \in F(x)$, two objective functions $f_1(x, s)$ and $f_2(x, s)$ will be used to evaluate the {\em quality} of the solution $s$.
The goal here is to {\em maximize} the objective functions. 
Meanwhile, as shown in \cite{papadimitriou2000approximability}, minimization problems can be treated similarly. 

The {\em Pareto curve} of an instance $x$, denoted by $P(x)$, is a set
of 2-dimension points.  For each $v \in P(x)$,
\begin{enumerate}[(1)]
\item There exists some $s \in F(x)$ such that $f_i(x, s) = v_i$, for $i
  = 1$ and $i = 2$.
\item There is no feasible solution $s'$ such that $f_i(x, s) \ge v_i$
  for $i = 1$ and $i = 2$, with at least one inequality holding
  strictly.
\end{enumerate}

The Pareto curve naturally provides a trade-off between the two
objective functions.  However, the Pareto curve is exponentially large
in size in general and cannot be efficiently computed.  Thus,
\cite{papadimitriou2000approximability} considered the approximate
version of Pareto curves. The {\em $\varepsilon$-approximate Pareto
  curve} of an instance $x$, denoted by $P_{\varepsilon}(x)$, is a set
of feasible solutions, such that for each feasible solution $s' \in
F(x)$, there exists some $s \in P_{\varepsilon}(x)$ such that $f_i(x,
s') \le (1 + \varepsilon) f_i(x, s)$ for $i = 1$ and $i = 2$.
For a problem $A$ where the objective functions are linear,
\cite{papadimitriou2000approximability} give an FPTAS for constructing
the approximate Pareto curve, given a pseudopolynomial algorithm for the
{\em exact version} of $A$. The exact version of $A$ is one where, given
an instance $x$ and a value $B$, we have to decide if there exists a
feasible solution with cost \emph{exactly} B.

Many combinatorial problems admit pseudopolynomial algorithms for the exact version, including the shortest path problem, the minimum spanning tree problem and the matching problem, as noted in \cite{papadimitriou2000approximability}. In the following sections, we will show how to efficiently implement the algorithm descried in previous sections, when the approximate Pareto curve of the underlying combinatorial problem of the $\combibandit$ instance can be computed by an FPTAS.
We also assume that the single-objective maximization version of the underlying combinatorial problem can be solved in polynomial time, i.e., given a weight vector $w$, there exists an algorithm that runs in polynomial time that can calculate $\argmax_{A \in \mathcal{F}} w(A)$.
\footnote{Such an algorithm has already been used implicitly in Line~\ref{line:effhatO} of algorithm \algeffi{}.}

\subsubsection{Efficient Implementation of $\maxoracle$, $\checksol$ and $\mathsf{Unique}$}

We begin with the implementation of $\maxoracle$. Notice that
$\maxoracle$ is actually a bi-objective optimization problem, by setting
$f_1(\cdot)$ to be $\mu(\cdot)$ and $f_2(\cdot)$ to be $w(\cdot)$.  We
can efficiently implement $\maxoracle$ by listing all points in the
approximate Pareto curve.  Notice that it is required that $\maxoracle$
outputs a solution with {\em additive} approximation term, while the
FPTAS presented in \cite{papadimitriou2000approximability} can only be
used to generate approximate Pareto curve with multiplicative
approximation ratio.  An important observation is that, for any $S \in
\mathcal{F}$, $\mu(S)$ is bounded by $O(n)$ and $w(S)$ is bounded by a
function polynomial in the sample complexity of our algorithm.  Thus, to
calculate an additive $\varepsilon$-approximate Pareto curve, it
suffices to set $\varepsilon'$ to be a value polynomial in $n,
\varepsilon$ and the sample complexity of our algorithm, and calculate
the multiplicative $\varepsilon'$-approximate Pareto curve.

Similarly, $\checksol$ is also a bi-objective optimization problem, by
setting $f_1(\cdot)$ to be $\meanhat{k}{}(\cdot)$ and $f_2(\cdot)$ to be
$\hat\mu(\cdot)$.  We can still efficiently implement $\checksol$ by
listing all points in the approximate Pareto curve. We omit
implementation details due to the similarity.

Given a polynomial-time algorithm $\mathbb{A}$ for the single-objective
maximization version of the underlying combinatorial problem, it will be
straightforward to implement $\mathsf{Unique}$. One possible way is to
calculate the subset with second largest objective value, which is given
as follows.  We first call $\mathbb{A}$ to find a subset $A$ with
maximum $\mu(A)$.  Then we enumerate every element $a \in A$, set
$\mu(a)$ to $-\infty$ and call $\mathbb{A}$ again.  By doing so, we will
be able to find the subset $A'$ with second largest objective value.  We
can then decide whether there is exactly one subset $A$ such that
$\mu(A) \ge \theta$ by comparing $\mu(A')$ with $\theta$.

\subsubsection{Efficient Implementation of $\mathsf{SimultEst}$ and $\mathsf{Verify}$}

Now we present our implementation for $\simest$. To solve the convex
program described in Algorithm \ref{algo:simest}, we apply the {\em
  Ellipsoid method}. It suffices to devise a polynomial time {\em
  separation oracle} \footnote{ Given a point $x$, the separation oracle
  needs to decide whether $x$ is in the feasible region. If not, the
  separation oracle should output a constraint that $x$ violates.  }
(see e.g., \cite{schrijver2002combinatorial}).  Concretely, we need to
solve the following separation problem.

\begin{defi}[Separation problem of $\simest$]
  Given $(\mu, \thetahi, \thetalo, \varepsilon, \delta)$ and vector
  $m^{\star}$, the goal of the separation problem of $\simest(\mu,
  \thetahi, \thetalo, \varepsilon, \delta)$ is to decide whether there
  exists two subsets $A, B \in \{ A' \in \mathcal{F}: \mu(A') \ge
  \thetahi\}$ such that
  $$
  \sum_{i \in A \oplus B} \frac{1}{m^{\star}_i} \ge
  \frac{\varepsilon^2}{2\ln(2 / \delta)}.
  $$
\end{defi}

Notice that we do not need to solve the separation problem exactly: a
constant approximation suffices, as this would only increase a constant
factor hidden in the big-O notation of the sample complexity.  (This
trick is often used in the approximation algorithms literature; see,
e.g., \cite{carr2000strengthening}.)  Specifically, it is sufficient to
find two subsets $A, B \in \{ A' \in \mathcal{F}: \mu(A') \ge
\thetahi\}$ such that
$$
\sum_{i \in A \oplus B} \frac{1}{m^{\star}_i} \ge C \cdot
\frac{\varepsilon^2}{2\ln(2 / \delta)}
$$
for some constant $C$
(assuming there are subsets $A',B'$ satisfying
$\sum_{i \in A' \oplus B'} \frac{1}{m^{\star}_i} \ge
\frac{\varepsilon^2}{2\ln(2 / \delta)}$). Moreover, as noted in the
previous section, the constraint $A, B \in \{ A' \in \mathcal{F}:
\mu(A') \ge \thetahi\}$ can also be relaxed to allow an additive
approximate term of $\thetahi - \thetalo$.

To decide whether such a pair of subset $(A, B)$ exists or not, we first
arbitrarily choose a subset $O$ in $\{ A' \in \mathcal{F}: \mu(A') \ge
\thetahi\}$ and then find a subset $O' \in \{ A' \in \mathcal{F}:
\mu(A') \ge \thetahi\}$ such that
$$
\sum_{i \in O \oplus O'} \frac{1}{m^{\star}_i}. 
$$
is maximized. The following lemma shows that, $(O, O')$ is a
2-approximation of the original separation problem.
\begin{Lemma}
  \label{lem:efficientapp1}
  For any $A, B \in \{ A' \in \mathcal{F}: \mu(A') \ge \thetahi\}$,
  $$
  \sum_{i \in O \oplus O'} \frac{1}{m^{\star}_i} \ge \frac{1}{2}\sum_{i
    \in A \oplus B} \frac{1}{m^{\star}_i}.
  $$
\end{Lemma}

\begin{proof}[Proof of Lemma \ref{lem:efficientapp1}]
As $O'$ is chosen so that
$$
\sum_{i \in O \oplus O'} \frac{1}{m^{\star}_i}. 
$$
is maximized, it follows that
$$
\sum_{i \in O \oplus O'} \frac{1}{m^{\star}_i} \ge \sum_{i \in O \oplus A} \frac{1}{m^{\star}_i}
$$
and
$$
\sum_{i \in O \oplus O'} \frac{1}{m^{\star}_i} \ge \sum_{i \in O \oplus B} \frac{1}{m^{\star}_i}.
$$
Thus,
$$
2\sum_{i \in O \oplus O'} \frac{1}{m^{\star}_i} \ge \sum_{i \in O \oplus A} \frac{1}{m^{\star}_i} +  \sum_{i \in O \oplus B} \frac{1}{m^{\star}_i} \ge \sum_{i \in A \oplus B} \frac{1}{m^{\star}_i}.
$$
\end{proof}
Now it remains to show how to find $O'$ efficiently. In order to find $O'$, we find $O_1\in \{ A' \in \mathcal{F}: \mu(A') \ge \thetahi\}$ such that
$$
\sum_{i \in O \backslash O_1} \frac{1}{m^{\star}_i} 
$$
is maximized, and $O_2\in \{ A' \in \mathcal{F}: \mu(A') \ge \thetahi\}$ such that
$$
\sum_{i \in O_2 \backslash O} \frac{1}{m^{\star}_i} 
$$
is maximized. 
Again, the following lemma shows that, by using the method described above, we can get a 2-approximation.
\begin{Lemma} \label{lem:efficientapp2}
For any $O' \in \{ A' \in \mathcal{F}: \mu(A') \ge \thetahi\}$, 
$$
2\max \left\{ \sum_{i \in O \backslash O_1} \frac{1}{m^{\star}_i}, \sum_{i \in O_2 \backslash O} \frac{1}{m^{\star}_i}  \right\}
\ge \sum_{i \in O \oplus O'} \frac{1}{m^{\star}_i} .
$$
\end{Lemma}
\begin{proof}[Proof of Lemma \ref{lem:efficientapp2}]
\begin{align*}
2\max \left\{ \sum_{i \in O \backslash O_1} \frac{1}{m^{\star}_i}, \sum_{i \in O_2 \backslash O} \frac{1}{m^{\star}_i}  \right\}
\ge \sum_{i \in O \backslash O_1}\frac{1}{m^{\star}_i} + \sum_{i \in O_2 \backslash O} \frac{1}{m^{\star}_i}.
\end{align*}
According to our choice of $O_1$ and $O_2$, it follows that
$$
\sum_{i \in O \backslash O_1}\frac{1}{m^{\star}_i} + \sum_{i \in O_2 \backslash O} \frac{1}{m^{\star}_i} \ge \sum_{i \in O \backslash O'}\frac{1}{m^{\star}_i} + \sum_{i \in O' \backslash O} \frac{1}{m^{\star}_i} =  \sum_{i \in O \oplus O'} \frac{1}{m^{\star}_i}.
$$
\end{proof}

The analysis above suggests, to decide whether there exists two subsets $A, B \in \{ A' \in \mathcal{F}: \mu(A') \ge \thetahi\}$ such that
$$
\sum_{i \in A \oplus B} \frac{1}{m^{\star}_i} \ge \frac{\varepsilon^2}{2\ln(2 / \delta)}
$$
{\em approximately}, it suffices to decide whether exists $O_1, O_2 \in \{ A' \in \mathcal{F}: \mu(A') \ge \thetahi\}$ such that
$$
\sum_{i \in O \backslash O_1} \frac{1}{m^{\star}_i} \ge \frac{\varepsilon^2}{2\ln(2 / \delta)}
$$
or
$$
\sum_{i \in O_2 \backslash O} \frac{1}{m^{\star}_i} \ge \frac{\varepsilon^2}{2\ln(2 / \delta)}.
$$

The problem of finding $O_1$ and $O_2$, are actually bi-objective optimization problems. As mentioned in previous sections, the first constraint, i.e., $O_1, O_2 \in \{ A' \in \mathcal{F}: \mu(A') \ge \thetahi\}$ can be relaxed to allow an additive approximate term of $\thetahi - \thetalo$, where $1 / (\thetahi - \thetalo)$ is bounded by the sample complexity of our algorithm. Thus, by using the approximate Pareto curve, it is straightforward to decide whether such $O_2$ exists or not. 
We set $w_i$ to be $\frac{1}{m^*_i}$, and further, for any $i \in O$, we set $w_i$ to be zero. 
We can then decide whether $O_2$ exists or not by calling $\maxoracle$ and using $w$ as the weight vector.

Deciding whether $O_1$ exists or not is more involved, but still in a similar manner. Again, our plan is to approximately decide the existence of such $O_1$.
More specifically, our method will return ``yes'' when there exists $O_1$ such that
$$
\sum_{i \in O \backslash O_1} \frac{1}{m^{\star}_i} \ge \frac{2\varepsilon^2}{2\ln(2 / \delta)},
$$
return ``no'' when for any $O_1 \in \{ A' \in \mathcal{F}: \mu(A') \ge \thetahi\}$,
$$
\sum_{i \in O \backslash O_1} \frac{1}{m^{\star}_i} \le \frac{\varepsilon^2}{2\ln(2 / \delta)}
$$
and return arbitrarily otherwise.

When $$\sum_{i \in O} \frac{1}{m^{\star}_i}  < \frac{2\varepsilon^2}{2\ln(2 / \delta)}$$ we simply return ``no'', as there will be no $O_1$ such that
$$
\sum_{i \in O \backslash O_1} \frac{1}{m^{\star}_i} \ge \frac{2\varepsilon^2}{2\ln(2 / \delta)}.
$$
Otherwise, we set $w_i$ to be $\frac{1}{m^{\star}_i}$, and further, for any $i \notin O$, we set $w_i$ to be zero.
Then we use approximate Pareto curve to find $O_1  \in \{ A' \in \mathcal{F}: \mu(A') \ge \thetahi\}$ with approximately maximum $w(O_1)$.

Here, we use set the multiplicative approximation ratio to be 
$$
1 + \frac{ \varepsilon^2}{2\ln(2 / \delta) w(O)},
$$
as such a multiplicative approximation ratio will induce an additive approximate term of
$$
\frac{ \varepsilon^2}{2\ln(2 / \delta) w(O)} \cdot w(O_1^*) \le \frac{ \varepsilon^2}{2\ln(2 / \delta) w(O)} \cdot w(O)= \frac{ \varepsilon^2}{2\ln(2 / \delta)},
$$
where $O_1^{*}$ denotes the subset in $\{ A' \in \mathcal{F}: \mu(A') \ge \thetahi\}$ with maximum 
$$
\sum_{i \in O \backslash O_1^*} \frac{1}{m^{\star}_i}.
$$
Such an additive approximate term is enough to distinguish the two cases (return ``yes'' or ``no'') stated above. 
Meanwhile, the time complexity for calculating such an approximate Pareto curve is bounded by a function polynomial in the sample complexity of our algorithm. 

Given the efficient implementation for $\simest$, $\verify$ can be
implemented in a similar manner. We also apply the Ellipsoid method and
approximately solve the separation problem by using approximate Pareto
curve. We do not repeat those details due to the similarity.


\section{Optimal Algorithm for General Sampling Problem}
\label{sec:general-algo}
        
In this section, we present a nearly optimal algorithm
$\alggen$ for the \generalbandit{} problem.  Given an instance $\inst =
(\armseq, \anssetcol)$ and a confidence level $\delta$, \alggen{} either
identifies an answer set in $\anssetcol$ as the answer, or reports an
error.  The algorithm is guaranteed to return the correct answer with
probability $1 - \delta - \delta_0$, where $\delta_0 = 0.01$, while the
probability of returning an incorrect answer is upper bounded by
$\delta$.  Therefore, \alggen{} can be transformed to a $\delta$-correct
algorithm while retaining its sample complexity by applying the parallel
simulation idea from Lemma \ref{lem:parasim}.

\subsection{Algorithm}
Algorithm \alggen{} consists of two stages.  In the first stage, we
sample each arm repeatedly in round-robin fashion, until the confidence
region of $\mean$ intersects exactly one answer set
$\candansset\in\anssetcol$.  \footnote{Here $\ltwoball(x, r)$ denotes
  the closed $\ell^{2}$-ball $\left\{x' \in \real^{n}: \twonorm{x-x'}\le
    r\right\}$.}  We identify $\candansset$ as the candidate answer.  We
further sample each arm a few more times in order to obtain a
sufficiently tight confidence region for the second stage.

The second stage is devoted to verifying the candidate $\candansset$.
We first calculate the optimal sampling profile by linear programming.
Let $\altans$ denote
$\bigcup_{\ansset\in\anssetcol\setminus\{\candansset\}}\ansset$, the
union of all answer sets other than $\candansset$.  Each point $\tilmean
\in \altans$ defines a constraint of the linear program, which states
that sufficiently many samples must be taken, in order to distinguish
the actual mean profile from $\tilmean$.  Finally, we verify the
candidate answer by sampling the arms according to the sampling profile.

Note that in the first stage, \alggen{} samples the arms in an inefficient
round-robin fashion, while the candidate answer $\candansset$ is verifed
using the optimal sampling profile in Stage 2.  Thus, \alggen{} uses a
less stringent confidence (i.e., $\delta_0$) in Stage 1, and then adopts
the required confidence level $\delta$ in the second stage.

\begin{algorithm2e}
  \caption{$\alggen(\inst, \delta)$}
  \KwIn{Instance $\inst = (\armseq, \anssetcol)$ and confidence level $\delta$.}
  \KwOut{Either an answer set in $\anssetcol$ or an error.}
  $t \leftarrow 0, \delta_0 \gets 0.01$\;
  \Repeat{ $\ltwoball\left(\empmean^{(t)},3r_t\right)$  intersects with
    exactly one of the answer sets, denoted by $\candansset$
  } {
    $t \leftarrow t + 1$\;
    Sample each arm in $\armseq$ once\;
    $\empmean^{(t)}\leftarrow$ empirical means of the arms among the
    first $t$ samples\; 
    $r_t \leftarrow \sqrt{\left[2n +
        3\ln\left(\delta_0/(4t^2)\right)^{-1}\right]/t}$\; 
  }
  $\alpha \leftarrow r_t/\sqrt{8n}$;
  $M \leftarrow \alpha^{-2}\left[2n + 3\ln(\delta_0/2)^{-1}\right]$\;
  $\empmean \leftarrow \Sample((M, M, \ldots, M))$\;
  \uIf{$\ltwoball(\empmean,r_t)$ intersects with $\altans$} {
    \textbf{return} error\;
  }
  Let $x^*$ be the optimal solution to the following linear program:
    \begin{equation}\label{eq:LP}\begin{split}
      \textrm{minimize}~~~~&\sum_{i=1}^{n}x_i\\
      \textrm{subject
        to}~~~~&\sum_{i=1}^{n}\left(\tilmean_i-\empmean_i\right)^2x_i
      \ge 1,~\forall \tilmean\in\altans,\\ 
      & x_i\ge 0\text{.}
    \end{split}
  \end{equation}

  $\beta\leftarrow64$;
  $m \leftarrow \beta x^*_i\left(\ln\delta^{-1}+n\right)$\;
  $X \leftarrow \Sample(m)$\;
  \uIf{$\sum_{i=1}^{n}m_i(X_i - \empmean_i)^2
    \le 36\left(\ln\delta^{-1}+n\right)$
  } {
    \textbf{return} $\candansset$\;
  } \uElse {
    \textbf{return} error\;
  }
\end{algorithm2e}

\subsection{Correctness}

\paragraph{Good Events.}
We start by defining two ``good events'' conditioning on which the
correctness and the sample complexity optimality of \alggen{} can be
guaranteed.  Recall that $\mean_i$ denote the mean of arm $\arm_i$.
Define $\goodevent_0$ as the event that in Stage 1, $\twonorm{\empmean^{(k)}
  - \mean}\le r_k$ holds for all $k$, and $\twonorm{\empmean -
  \mean}\le\alpha$ also holds.  Note that both
$k\twonorm{\empmean^{(k)}-\mean}^2$ and $M\twonorm{\empmean-\mean}^2$
are $\chi^2$ random variables with $n$ degrees of freedom.  The tail
probability bound of the $\chi^{2}$-distribution (Lemma
\ref{lem:chi2bound}) implies that
$$
\pr{\twonorm{\empmean^{(k)}-\mean} > r_k} =
\pr{k\twonorm{\empmean^{(k)}-\mean}^2 > 2n +
  3\ln\left(\frac{\delta_0}{4k^2}\right)^{-1}} \le
\frac{\delta_0}{4k^2}\text{.}
$$
Similarly,
$$
\pr{\twonorm{\empmean-\mean} > \alpha} = \pr{M\twonorm{\empmean-\mean}^2
  > 2n + 3\ln\left(\frac{\delta_0}{2}\right)^{-1}} \le
\frac{\delta_0}{2}\text{.}
$$
By a union bound,
$$
\Pr[\goodevent_0] \ge 1 - \frac{\delta_0}{2} -
\sum_{k=1}^{\infty}\frac{\delta_0}{4k^2} \ge 1 - \delta_0\text{.}
$$
We define $\goodevent$ as the event that in Stage 2, it holds that
\begin{equation}\label{eq:eventone}
  \sum_{i=1}^{n} m_i(X_i - \mean_i)^2
  \le 2n + 3\ln\delta^{-1}\text{.}
\end{equation}
Note that $\sqrt{m_i}(X_i - \mean_i)$ follows the standard normal
distribution. Thus, Lemma \ref{lem:chi2bound} implies that that
$\Pr[\goodevent]\ge 1 - \delta$.

\paragraph{LP Solution Bound.} We have the following simple lemma, which
upper bounds the optimal solution $x^*$ of the linear program in Stage
2.

\begin{Lemma}\label{lem:LPbound}
  $\sum_{i=1}^{n}x^*_{i}\le nr_t^{-2}$.
\end{Lemma}

\begin{proof}
  Since \alggen{} completes Stage 1 without reporting an error,
  $\ltwoball(\empmean, r_t)$ is disjoint from $\altans$.  In other
  words, for all $\tilmean\in\altans$ we have $\twonorm{\tilmean -
    \empmean} > r_t$.  It directly follows that
  $$x_1 = x_2 = \cdots = x_n = r_t^{-2}$$
  is a feasible solution of the linear program \eqref{eq:LP}, which
  proves the lemma.
\end{proof}

\begin{Lemma}[Soundness]
  \label{lem:sound}
  Conditioning on $\goodevent$, \alggen{} never returns an incorrect
  answer.
\end{Lemma}

\begin{proof}
  Recall that at the end of algorithm \alggen{}, the following
  inequality is verified:
  \begin{equation}\label{eq:verify}
    \sum_{i=1}^{n}m_i(X_i - \empmean_i)^2
    \le 36\left(\ln\delta^{-1}+n\right)\text{.}
  \end{equation}
  Suppose the candidate answer $\candansset$ chosen in Stage 1 is
  correct, the lemma trivially holds, so assume that the candidate is
  incorrect (i.e., $\mean\in\altans$). We now show that conditioning on
  event $\goodevent$, inequality \eqref{eq:verify} is violated, and thus
  \alggen{} reports an error, rather than returning the incorrect answer
  $\candansset$.
  
  Define $a_i = \sqrt{m_i}\left(X_i - \empmean_i\right)$. Note that
  inequality \eqref{eq:verify} is equivalent to
  $\twonorm{a}\le6\sqrt{n+\ln\delta^{-1}}$.
  Let us write $a$ into $a = b + c$, where $b_i = \sqrt{m_i}\left(X_i -
    \mean_i\right)$ and $c_i = \sqrt{m_i}\left(\mean_i -
    \empmean_i\right)$.  We first note that conditioning on event
  $\goodevent$, inequality \eqref{eq:eventone} guarantees
  that $$\twonorm{b} = \sqrt{\sum_{i=1}^{n}m_i\left(X_i -
      \mean_i\right)^2}\le \sqrt{2n + 3\ln\delta^{-1}} < 2\sqrt{n +
    \ln\delta^{-1}}\text{.}$$ On the other hand, since
  $\mean\in\altans$, the constraint corresponding to point $\mean$ in
  linear program \eqref{eq:LP} implies that
  $$\twonorm{c} = \sqrt{\sum_{i=1}^{n}m_i\left(\mean_i -
      \empmean_i\right)^2} =
  \sqrt{\beta\left(n+\ln\delta^{-1}\right)\sum_{i=1}^{n}x^*_i\left(\mean_i
      - \empmean_i\right)^2}\ge 8\sqrt{n + \ln\delta^{-1}}\text{.}$$ 
  Therefore, we conclude that
  $$\twonorm{a} = \twonorm{b+c}\ge\twonorm{c}-\twonorm{b}>6\sqrt{n+\ln\delta^{-1}}\text{,}$$ 
  which completes the proof.
\end{proof}

\begin{Lemma}[Completeness]
  \label{lem:complete}
  Conditioning on $\goodevent_0\cap\goodevent$, \alggen{} always returns
  the correct answer.
\end{Lemma}
                
\begin{proof}
  Recall that conditioning on event $\goodevent_0$, the actual mean profile
  $\mean$ is in $\ltwoball(\empmean^{(t)}, r_t)$. According to
  \alggen{}, $\candansset$ is the only answer set that intersects
  $\ltwoball(\empmean^{(t)}, r_t)$, and thus $\candansset$ is indeed the
  correct answer. It remains to show that \alggen{} terminates without
  reporting errors.

  We first prove that at the end of Stage 1, $\ltwoball(\empmean, r_t)$
  and $\altans$ are disjoint. Let $\tilmean$ be an arbitrary point in
  $\altans$. Our choice of $t$ ensures that $\twonorm{\empmean^{(t)} -
    \tilmean}\ge 3r_t\text{.}$ Conditioning on event $\goodevent_0$, we also
  have $\twonorm{\empmean^{(t)} - \mean}\le r_t$ and $\twonorm{\empmean
    - \mean}\le\alpha < r_t\text{.}$ It follows from the three
  inequalities above that
  $$\twonorm{\empmean - \tilmean} \ge \twonorm{\empmean^{(t)} -
    \tilmean} - \twonorm{\empmean^{(t)} - \mean} - \twonorm{\empmean -
    \mean} > 3r_t - r_t - r_t = r_t\text{,}$$ which implies
  $\tilmean\not\in\ltwoball(\empmean, r_t)$. Therefore,
  $\ltwoball(\empmean, r_t)$ and $\altans$ are disjoint, and \alggen{}
  finishes Stage 1 without reporting an error.

  Next we show that \alggen{} does not report an error at the end of
  Stage 2 (i.e., inequality \eqref{eq:verify} holds). As in the proof of
  Lemma~\ref{lem:sound}, define $a_i = \sqrt{m_i}\left(X_i -
    \empmean_i\right)$, $b_i = \sqrt{m_i}\left(X_i - \mean_i\right)$,
  and $c_i = \sqrt{m_i}\left(\mean_i-\empmean_i\right)$. Then inequality
  \eqref{eq:verify} is equivalent to showing
  $\twonorm{a} = \twonorm{b+c} \le6\sqrt{n+\ln\delta^{-1}}$.
  Conditioning on $\goodevent$, $\twonorm{b} < 2\sqrt{n+\ln\delta^{-1}}$
  follows from inequality \eqref{eq:eventone} as in Lemma
  \ref{lem:sound}. Next,
  \begin{align*}
    \twonorm{c}^2 &=
    \beta\left(n+\ln\delta^{-1}\right)\sum_{i=1}^{n}x^*_i\left(\mean_i -
      \empmean_i\right)^2 \tag{\text{by definition of $c$ and $m_i$}}\\
    &\le
    64\left(n+\ln\delta^{-1}\right)\bigg(\sum_{i=1}^{n}x^*_i\bigg)\sum_{i=1}^{n}\left(\mean_i
      - \empmean_i\right)^2 \tag{$\vec{u}\cdot \vec{v} \leq
      \|\vec{u}\|_1\cdot\|\vec{v}\|_1$}  \\ 
    &\le 64\left(n+\ln\delta^{-1}\right)\cdot nr_t^{-2}\cdot \alpha^2\\
    &\le 8\left(n+\ln\delta^{-1}\right)\text{.}
  \end{align*}
  Above, the third step applies Lemma~\ref{lem:LPbound} and the fact
  that $\twonorm{\empmean-\mean}\le\alpha$. The last step plugs in the
  parameter $\alpha = r_t/\sqrt{8n}$.  Therefore, we conclude that
  $$\twonorm{a}\le\twonorm{b} + \twonorm{c} < 2\sqrt{n+\ln\delta^{-1}} + \sqrt{8(n+\ln\delta^{-1})} < 6\sqrt{n+\ln\delta^{-1}}\text{,}$$
  and thus \alggen{} returns the correct answer without reporting an error.
\end{proof}

\subsection{Sample Complexity}
We show that \alggen{} is nearly optimal: the sample complexity of the algorithm matches the instance lower bound $\Omega(\gklow(\inst)\ln\delta^{-1})$ as $\delta$ tends to zero.  Specifically, we give an
upper bound on the sample complexity of algorithm \alggen{} conditioning
on the ``good event'' $\goodevent_0$, in terms of $\gklow(\inst)$ and
$$\Delta = \inf_{\tilmean\in\alt(\ansset)}\twonorm{\mean-\tilmean}\text{.}$$
(Note that the assumption that $\ansset$ is disjoint from the closure of
$\alt(\ansset)$ guarantees that $\Delta > 0$.) We first prove a simple
lemma, which relates $\Delta$ to $\gklow(\inst)$.

\begin{Lemma}\label{lem:DeltavsLow}
  $\gklow(\inst) \ge \Delta^{-2}$.
\end{Lemma}

\begin{proof}
  By definition, $\gklow(\inst) = \sum_{i=1}^{n}\tau^*_i$, where
  $\{\tau^*_i\}$ is the optimal solution to \eqref{eq:genlbdef}.  Note
  that
  $$
  \gklow(\inst)\twonorm{\tilmean_i-\mean_i}^2
  \ge \sum_{i=1}^{n}(\tilmean_i-\mean_i)^2\tau^*_i
  \ge 1\text{.}
  $$
  Thus
  $$
  \gklow(\inst) \ge
  \sup_{\tilmean\in\alt(\ansset)}\twonorm{\tilmean-\mean}^{-2} =
  \Delta^{-2}\text{.}
  $$
\end{proof}

\begin{Lemma}\label{lem:sample}
  Conditioning on event $\goodevent_0$, \alggen{} takes
  $O(\gklow(\inst)(\ln\delta^{-1} + n^3 + n\ln\Delta^{-1}))$ samples.
\end{Lemma}

\begin{proof}
  Recall that in the first stage of \alggen{}, $r_k$ is defined as
  $$r_k = \sqrt{\frac{2n+3\ln[\delta_0/(4k^2)]^{-1}}{k}}\text{,}$$
  and the number of samples taken in Stage 1 is $nt + nM$, where 
  $$M = \alpha^{-2}\left[2n + 3\ln(\delta_0/2)^{-1}\right]
  = O(n\alpha^{-2}) = O(n^2r_t^{-2})\text{,}$$ and $t$ is the smallest
  index such that $\ltwoball(\empmean^{(t)}, 3r_t)$ intersects only one
  set in $\anssetcol$. In order to upper bound $t$, let $t^*$ be the
  smallest integer such that $r_{t^*} < \Delta/4$. A simple calculation
  gives $t^* = O\left(\Delta^{-2}(n + \ln\Delta^{-1})\right)$. Moreover,
  at round $t^*$, conditioning on event $\goodevent_0$ implies that
  $\empmean^{(t^*)}\in\ltwoball(\mean, r_{t^*})$.  It follows that
  $$\ltwoball(\empmean^{(t^*)}, 3r_{t^*})
  \subseteq\ltwoball(\mean, 4r_{t^*}) \subset\ltwoball(\mean,
  \Delta)\text{.}$$ By definition of $\Delta$,
  $\ltwoball(\empmean^{(t^*)}, 3r_{t^*})$ is disjoint from
  $\alt(\ansset)$, and thus the round-robin sampling in Stage 1
  terminates before taking $t^*$ samples from each arm (i.e., $t\le
  t^*$). Therefore, $nt$ is upper bounded by
  $$nt^* = O\left(\Delta^{-2}(n^2+n\ln\Delta^{-1})\right)\text{.}$$
  Moreover, we have $nM = O(n^3r_t^{-2}) = O(n^3r_{t^*}^{-2}) =
  O(n^3\Delta^{-2})$. Also, we note that $\Delta^{-2}=O(\gklow(\inst))$
  by Lemma \ref{lem:DeltavsLow}. Putting this all together, the number
  of samples in Stage 1 is
  $$O\left(\gklow(\inst)(n^3+\ln\Delta^{-1})\right)\text{.}$$

  Now for the second stage samples. Let $\tau^*$ be the optimal solution
  to the linear program defined in \eqref{eq:genlbdef}. By definition,
  $\gklow(\inst) = \sum_{i=1}^{n}\tau^*_i$. Then we construct a feasible
  solution to the linear program in Stage 2 from $\tau^*$. Recall that
  conditioning on event $\goodevent_0$, we have $|\empmean_i -
  \mean_i|\le\twonorm{\empmean - \mean} \le \alpha$ for all
  $i\in[n]$. It follows that for all $\tilmean\in\altans$,
  $$\left(\tilmean_i - \empmean_i\right)^2 = [\left(\tilmean_i - \mean_i\right) + \left(\mean_i - \empmean_i\right)]^2\ge \left(\tilmean_i - \mean_i\right)^2/2 - 2\left(\mean_i - \empmean_i\right)^2\ge\left(\tilmean_i - \mean_i\right)^2/2 - 2\alpha^2\text{.}$$
  Here the second step applies the inequality $(a+b)^2\ge a^2/2 - 2b^2$.
  Therefore, for all $\tilmean\in\altans$,
  \begin{equation}\label{eq:sample}\begin{split}
      \sum_{i=1}^{n}\left(\tilmean_i - \empmean_i\right)^2\tau^*_i
      &\ge \sum_{i=1}^{n}\tau^*_i\left[\left(\tilmean_i - \mean_i\right)^2/2 - 2\alpha^2\right]\\
      &\ge \frac{1}{2}\sum_{i=1}^{n}\left(\tilmean_i - \mean_i\right)^2\tau^*_i - 2\alpha^2\sum_{i=1}^{n}\tau^*_i\\
      &\ge \frac{1}{2} - 2\alpha^2\cdot nr_t^{-2}
      =\frac{1}{4}\text{.}
    \end{split}\end{equation}
  The third step holds due to the feasibility of $\tau^*$ and the fact
  that $\sum_{i=1}^{n}\tau^*_i \le nr_t^{-2}$, which follows from an
  analogous argument to the proof of Lemma \ref{lem:LPbound}. The last
  step follows from our choice of parameter $\alpha = r_t/\sqrt{8n}$.

  Inequality \eqref{eq:sample} implies that $x_i = 4\tau^*_i$ is a
  feasible solution of the linear program in Stage 2 of \alggen{}. It
  follows that the number of samples taken in Stage 2 is bounded by
  \begin{equation*}\begin{split}
      \sum_{i=1}^{n}m_i
      =   \beta\left(\ln\delta^{-1}+n\right)\sum_{i=1}^{n} x^*_i
      \le \beta\left(\ln\delta^{-1} + n\right)\sum_{i=1}^{n}4\tau^*_i
      =   O\left(\gklow(\inst)\left(\ln\delta^{-1}+n\right)\right)\text{.}
    \end{split}\end{equation*}
  In conclusion, \alggen{} takes
  $$O\left(\gklow(\inst)(\ln\delta^{-1}+n^3 + n\ln\Delta^{-1})\right)$$
  samples in Stage 1 and Stage 2 in total, conditioning on event $\goodevent_0$.
\end{proof}

Finally, we prove Theorem~\ref{theo:general-upperb}, which we restate for convenience in the following.

\noindent\textbf{Theorem~\ref{theo:general-upperb}} (restated)\textit{
  There is a $\delta$-correct algorithm for $\generalbandit$ that takes
  $$
  O(\gklow(\inst)(\ln\delta^{-1} + n^3 + n\ln\Delta^{-1}))
  $$
  samples on any instance $\inst = (\armseq, \anssetcol)$ in expectation, where
    $$\Delta = \inf_{\tilmean\in\alt(\ansset)}\twonorm{\mean-\tilmean}$$
  is defined as the minimum Euclidean distance between the mean profile $\mean$ and an alternative mean profile $\tilmean\in\alt(\ansset)$ with an answer other than $\ansset$.
}

\begin{proof}
  By Lemmas \ref{lem:sound}, \ref{lem:complete} and \ref{lem:sample}, $\alggen$ is a $(\delta_0, \delta, A, B)$-correct algorithm for \generalbandit{} (as per Definition~\ref{def:alg}), where $\event_0 = \goodevent_0\cap\goodevent$, $\event_1=\goodevent$, $\delta_0 = 0.01$, $A = \gklow(\inst)$ and $B=\gklow(\inst)(n^3+n\ln\Delta^{-1})$. Lemma~\ref{lem:parasim} implies that there is a $\delta$-correct algorithm for \generalbandit{} with expected sample complexity
    $$O(A\ln\delta^{-1}+B) = O\left(\gklow(\inst)(\ln\delta^{-1}+n^3+n\ln\Delta^{-1})\right)\text{.}$$
\end{proof}


\bibliography{team} 

\newpage

\appendix

\section*{Organization of the Appendix}
In Appendices \ref{app:parasim}~and~\ref{app:efficientmiss}, we present the missing proofs in Sections \ref{sec:inefficient}~and~\ref{sec:efficient}. In Appendices \ref{sec:lower_bound_comb}~and~\ref{sec:anotherlb}, we prove our negative results on the sample complexity of \combibandit{} and \generalbandit{} (Theorems \ref{theo:worst-case-lowb-comb}~and~\ref{theo:lower-bound-general}).

\section{Missing Proof in Section~\ref{sec:inefficient}}\label{app:parasim}
In this section, we prove the ``parallel simulation'' lemma (Lemma~\ref{lem:parasim}) in Section~\ref{sec:inefficient}, which we restate below for convenience.

\noindent\textbf{Lemma~\ref{lem:parasim}} (restated) \textit{
	If there is a $(\delta_0, \delta, A, B)$ algorithm for a sampling problem for $\delta_0 = 0.01$ and any $\delta < 0.01$, there is also a $\delta$-correct algorithm for any $\delta < 0.01$ that takes $O(A\ln\delta^{-1} + B)$ samples in expectation.
}

\begin{proof}[Proof of Lemma~\ref{lem:parasim}]
For each integer $k \ge 0$, let $\alg_k$ be a $(\delta_0, \delta/2^{k+1}, A, B)$ algorithm for the problem. We construct an algorithm $\alg$, which simulates the sequence $\{\alg_k\}_{k\ge0}$ of algorithms in parallel.

We number the time slots with positive integers $1,2,\ldots$ At time slot $t$, for each integer $k\ge0$ such that $2^k$ divides $t$, $\alg$ either starts or resumes the execution of algorithm $\alg_k$, until $\alg_k$ requests a sample or terminates. In the former case, $\alg$ draws a sample from the arm that $\alg_k$ specifies and feeds it to $\alg_k$. After that, the execution of $\alg_k$ is suspended. As soon as some algorithm $\alg_k$ terminates without an error (i.e., it indeed returns an answer), $\alg$ outputs the answer that $\alg_k$ returns.

To analyze this construction, we let $\event_{0,k}$ and $\event_{1,k}$ denote the events $\event_0$ and $\event_1$ in Definition~\ref{def:alg} for algorithm $\alg_k$. By definition,
	$$\pr{\event_{0,k}}\ge 1 - \delta_0 - \delta/2^{k+1} \ge 0.98$$
and
	$$\pr{\event_{1,k}}\ge 1 - \delta/2^{k+1}\text{.}$$

We first note that since $\alg$ never returns an incorrect answer conditioning on $\bigcap_{k=0}^{\infty}\event_{1, k}$, by a union bound, the probability that $\alg$ outputs an incorrect answer is upper bounded by
	$$\sum_{k=0}^{\infty}\pr{\overline{\event_{1,k}}}\le\sum_{k=0}^{\infty}\delta/2^{k+1} = \delta\text{,}$$
and thus $\alg$ is $\delta$-correct.~\footnote{We may easily verify that $\alg$ terminates almost surely conditioning on $\bigcap_{k=0}^{\infty}\event_{1, k}$.}

Then we analyze the sample complexity of $\alg$. Let random variable $T$ be the smallest index such that $\event_{0,T}$ happens. Since the execution of the algorithm sequence $\{\alg_k\}$ is independent, we have
	$$\pr{T = k}
	\le \prod_{j=0}^{k-1}\left(1 - \pr{\event_{0,k}}\right)
	\le 0.02^k\text{.}$$
Conditioning on $T = k$, $\alg_k$ takes at most
	$\alpha(A \ln(2 ^ {k+1} / \delta) + B)$
samples before it terminates for some universal constant $\alpha$. Since $\alg_k$ takes a sample every $2^k$ time slots, algorithm $\alg$ terminates within $\alpha2^k(A\ln(2^{k+1}/\delta)+B)$ time steps. Then the total number of samples taken by $\alg$ is at most
	$$\sum_{j=0}^{\infty}\frac{\alpha2^k(A\ln(2^{k+1}/\delta)+B)}{2^j}
	\le \alpha2^{k+1}(A\ln(2^{k+1}/\delta)+B)\text{.}$$

Therefore, the expected number of samples taken by $\alg$ is upper bounded by
	\begin{equation*}\begin{split}
		&\sum_{k=0}^{\infty}\pr{T=k}\cdot\alpha2^{k+1}(A\ln(2^{k+1}/\delta)+B)\\
	&\le \alpha\sum_{k=0}^{\infty}0.02^k\cdot2^{k+1}(A\ln2^{k+1} + A\ln\delta^{-1}+B)\\
	&\le \alpha(A\ln\delta^{-1}+B)\sum_{k=0}^{\infty}0.02^k\cdot2^{k+1} +
		\alpha A\sum_{k=0}^{\infty}0.02^k\cdot2^{k+1} \ln2^{k+1}\\
	&=	O\left(A\ln\delta^{-1}+B\right)\text{.}
	\end{split}\end{equation*}
\end{proof}

\section{Missing Proofs in Section~\ref{sec:efficient}}\label{app:efficientmiss}
In this section, we present the missing proofs of Lemmas \ref{lem:efficorrect}~and~\ref{lem:effisample} in Section~\ref{sec:efficient}.

\subsection{Correctness}

We restate Lemma~\ref{lem:efficorrect} in the following.

\noindent\textbf{Lemma~\ref{lem:efficorrect}} (restated) \textit{
  For any $\delta\in(0, 0.01)$ and \combibandit{} instance $\combiband$,
  $\algeffi(\combiband, \delta)$ returns the correct answer with
  probability $1 - \delta_0 - \delta$, and returns an incorrect answer
  w.p.\ at most~$\delta$.
}

\begin{proof}
Define $\{\subsetfam_r\}$ and $\{\tilfam_r\}$ as
$$\subsetfam_{r+1} \coloneqq \{A\in\subsetfam:\meanhat{r}{}(A)\ge\theta_r\}$$
and
$$\tilfam_{r+1} \coloneqq \{A\in\subsetfam:\meanhat{r}{}(A)\ge\theta_r-\epsilon_r/\lambda\}\text{.}$$
Intuitively, $\subsetfam_r$ is analogous to the set $\subsetfam_r$ used
in \algnaive{}, which represents the collection of remaining sets at the
beginning of round $r$, while $\tilfam_{r+1}$ is a relaxed version of
$\subsetfam_{r+1}$.

Then we note that at each round $r$, when $\simest$ is called with
parameters $\mu = \meanhat{k-1}{}$ and $\thetahi =
\theta_{k-1}-\epsilon_{k-1}/\lambda$ for $k\in[r]$, the set
$\{A'\in\subsetfam:\mu(A')\ge\thetahi\}$ involved in the mathematical
program is exactly $\tilfam_k$. Similarly, when $\verify$ is called at
the last round with parameters $\thetahi_k = \theta_k -
\epsilon_k/\lambda$, the set
$\{A'\in\subsetfam:\meanhat{k-1}{}(A')\ge\thetahi_{k-1}\}$ is also
identical to $\tilfam_k$.

\paragraph{Good events.} As in the analysis of the \algnaive{}
algorithm, two good events $\goodevent_0$ and $\goodevent$ play
important roles. Let $\goodevent_{0,r}$ denote the event that either the
algorithm terminates before or at round $r$, or it holds that
$$\left|(\meanhat{r}{}(A)-\meanhat{r}{}(B))-(\mu(A)-\mu(B))\right| < \epsilon_k/\lambda$$
for all $k\in[r]$ and $A, B\in\tilfam_k$. Note that this definition is
stronger than the one in the analysis of \algnaive{}, where only the
accuracy of the gaps between set pairs in $\tilfam_r$ is required.
$\goodevent_0$ is defined as the intersection of all
$\goodevent_{0,r}$'s. Moreover, we define $\goodevent$ as the event that
at line~\ref{line:effhatO}, it holds that
$$\left|(\hat\mu(\hat O)-\hat\mu(A))-(\mu(\hat O)-\mu(A))\right|<\epsilon_k/\lambda$$
for all $k$ and $A\in\tilfam_k$.

The following lemma, similar to Lemma~\ref{lem:good-event-prob}, bounds the probability of the good events.
\begin{Lemma}\label{lem:effi-good-event-prob}
  $\pr{\goodevent_0} \geq 1-\delta_0$ and $\pr{\goodevent} \geq 1-\delta$.
\end{Lemma}
\begin{proof}
  Recall that $\numsamp{r}{}$ is the sum of
  $$\simest(
  \meanhat{k-1}{},
  \theta_{k-1}-\epsilon_{k-1}/\lambda,
  \theta_{k-1}-2\epsilon_{k-1}/\lambda,
  \epsilon_k/\lambda,
  \delta_r
  )$$
  over all $k\in[r]$. This guarantees that $\numsamp{r}{}$ is a valid solution to all programs. Specifically, for each $k\in[r]$ and $A,B\in\tilfam_k$, it holds that
  $$\sum_{i\in A\oplus B}1/\numsamp{r}{i}\le\frac{(\epsilon_k/\lambda)^2}{2\ln(2/\delta_r)}\text{.}$$
  By Lemma~\ref{lem:sum_dev}, it holds that
  $$\pr{\left|(\meanhat{r}{}(A)-\meanhat{r}{}(B))-(\mu(A)-\mu(B))\right| < \epsilon_k/\lambda}\ge1-\delta_r\text{.}$$
  A union bound over all possible choices of $k, A, B$ yields that
  $$\pr{\goodevent_{0,r}}\ge1-r|\subsetfam|^2\delta_r\ge1-\frac{\delta_0}{10r^2}\text{.}$$
  It follows from another union bound over all $r$ that $\pr{\goodevent_0} \ge 1-\delta_0$, and a similar argument proves that $\pr{\goodevent} \ge 1 - \delta$.
\end{proof}

\paragraph{Implications.} We prove the analogues of Lemmas \ref{lem:survive}~and~\ref{lem:chain} for \algeffi{}.
\begin{Lemma}\label{lem:effisurvive}
  Conditioning on $\goodevent_0$, $O\in\subsetfam_r\subseteq\tilfam_r$ for all $r$.
\end{Lemma}
\begin{proof}[Proof of Lemma~\ref{lem:effisurvive}]
  Suppose for a contradiction that
  $O\in\subsetfam_r\setminus\subsetfam_{r+1}$ for some $r$. Recall that
  $\maxoracle$ guarantees $\opt_r=\meanhat{r}{}(A)$ for some
  $A\in\subsetfam$ such that
  $\meanhat{r-1}{}(A)\ge\theta_{r-1}-\epsilon_{r-1}/\lambda$, i.e.,
  $A\in\tilfam_r$. Since $A\in\tilfam_r$ and
  $O\in\subsetfam_r\subseteq\tilfam_r$, it holds conditioning on
  $\goodevent_0$ that
  $$\meanhat{r}{}(O)-\meanhat{r}{}(A) > \mu(O)-\mu(A)-\epsilon_r/\lambda \ge -\epsilon_r/\lambda\text{.}$$
  It follows that
  $$\meanhat{r}{}(O) > \meanhat{r}{}(A) - \epsilon_r/\lambda = \opt_r-\epsilon_r/\lambda > \theta_r\text{,}$$
  which leads to a contradiction to $O\notin\subsetfam_{r+1}$.
\end{proof}

The following lemma is analogous to Lemma~\ref{lem:chain}. In addition,
we also characterize the relation between $G_{\ge r-1}$ and
$\{A\in\subsetfam:\meanhat{r}{}(A)\ge\theta_r-2\epsilon_r/\lambda\}$,
which serves as a further relaxation of $\tilfam_r$.
\begin{Lemma}\label{lem:effichain}
  Conditioning on $\goodevent_0$, it holds that
  $G_{\ge r}\supseteq\tilfam_{r+1}\supseteq\subsetfam_{r+1}\supseteq G_{\ge r+1}$
  and
  $G_{\ge r-1}\supseteq\{A\in\subsetfam:\meanhat{r}{}(A)\ge\theta_r-2\epsilon_r/\lambda\}$.
\end{Lemma}
\begin{proof}[Proof of Lemma~\ref{lem:effichain}]
  We prove by induction on $r$. The base case that $r = 0$ holds due to
  the observation that $\subsetfam_1=\tilfam_1=G_{\ge
    0}=\subsetfam$. Now we prove the lemma for $r\ge1$.

  \paragraph{Part I. $A\notin G_{\ge r}$ implies $A\notin\tilfam_{r+1}$.} Suppose that $A\in G_k$ for some $k\le r-1$. Then by the inductive hypothesis, $A\in G_{\ge k}\subseteq\tilfam_k$. Also by Lemma~\ref{lem:effisurvive}, $O\in\subsetfam_k\subseteq\tilfam_k$. Therefore, conditioning on event $\goodevent_0$, it holds that
  $$\meanhat{r}{}(O)-\meanhat{r}{}(A)
  > \mu(O)-\mu(A)-\epsilon_k/\lambda
  > \epsilon_{k+1}-\epsilon_k/\lambda
  =(1/2-1/\lambda)\epsilon_k\text{.}$$
  Recall that $\maxoracle$ guarantees that
  $$\opt_r - \meanhat{r}{}(O)
  \ge\max_{B\in\subsetfam_r}\meanhat{r}{}(B)-\epsilon_r/\lambda - \meanhat{r}{}(O)
  \ge-\epsilon_r/\lambda\text{.}$$
  The second step holds since, by Lemma~\ref{lem:effisurvive}, $O\in\subsetfam_r$.
  Note that as $k\le r-1$, $\epsilon_k \ge 2\epsilon_r$, and then the two inequalities above imply that
  $$\opt_r-\meanhat{r}{}(A)
  >(1/2-1/\lambda)\epsilon_k -\epsilon_r/\lambda
  \ge(1-3/\lambda)\epsilon_r\text{.}$$
  It follows that
  $$\meanhat{r}{}(A) < \opt_r - (1-3/\lambda)\epsilon_r \le \opt_r - (1/2 + 3/\lambda)\epsilon_r = \theta_r - \epsilon_r/\lambda\text{,}$$
  and thus $A\notin\tilfam_{r+1}$. Here the second holds due to $\lambda = 20$.
  
  \paragraph{Part II. $A\in G_{\ge r+1}$ implies
    $A\in\subsetfam_{r+1}$.} For fixed $A\in G_{\ge r+1}$, the inductive
  hypothesis implies that $A \in G_{\ge r}\subseteq\tilfam_r$. Also by
  Lemma~\ref{lem:effisurvive}, $O \in \tilfam_r$. Thus conditioning on
  event $\goodevent_0$,
  $$\meanhat{r}{}(O)-\meanhat{r}{}(A) < \mu(O)-\mu(A) + \epsilon_r/\lambda \le \epsilon_{r+1}+ \epsilon_r/\lambda = (1/2+1/\lambda)\epsilon_r\text{.}$$
  Note that $\maxoracle$ guarantees that $\opt_r = \meanhat{r}{}(B)$ for some $B\in\tilfam_r$. Thus, conditioning on $\goodevent_0$,
  $$\opt_r-\meanhat{r}{}(O)
  =\meanhat{r}{}(B)-\meanhat{r}{}(O)
  <\mu(B)-\mu(O)+\epsilon_r/\lambda
  \le\epsilon_r/\lambda\text{.}$$
  Adding the two inequalities above yields
  $$\meanhat{r}{}(A) > \opt_r - (1/2+2/\lambda)\epsilon_r = \theta_r\text{,}$$
  and therefore $A\in\subsetfam_{r+1}$.
  
  \paragraph{Part III. $A\notin G_{\ge r-1}$ implies $\meanhat{r}{}(A)<\theta_r-2\epsilon_r/\lambda$.} Suppose $A\in G_k$ for $k\le r-2$. By the inductive hypothesis, $A\in G_{\ge k}\subseteq\tilfam_k$. Since $O\in\tilfam_k$ by Lemma~\ref{lem:effisurvive},
  $$\meanhat{r}{}(O)-\meanhat{r}{}(A)
  >\mu(O)-\mu(A)-\epsilon_k/\lambda
  >\epsilon_{k+1}-\epsilon_k/\lambda
  =(1/2-1/\lambda)\epsilon_k\text{.}$$
  The specification of $\maxoracle$ guarantees that
  $$\opt_r - \meanhat{r}{}(O)
  \ge\max_{B\in\subsetfam_r}\meanhat{r}{}(B)-\epsilon_r/\lambda - \meanhat{r}{}(O)
  \ge-\epsilon_r/\lambda\text{.}$$
  Recall that $k\le r-2$ and $\lambda=20$. Therefore,
  \begin{equation*}\begin{split}
      \meanhat{r}{}(A)
      &<\opt_r-(1/2-1/\lambda)\epsilon_k+\epsilon_r/\lambda\\
      &=\opt_r-(2-4/\lambda)\epsilon_r+\epsilon_r/\lambda\\
      &<\opt_r-(1/2+2/\lambda)\epsilon_r-2\epsilon_r/\lambda
      =\theta_r-2\epsilon_r/\lambda\text{.}
    \end{split}\end{equation*}
\end{proof}

\paragraph{Correctness conditioning on $\goodevent_0\cap\goodevent$.} We
show that \algeffi{} always returns the correct answer $O$ conditioning
on both $\goodevent_0$ and $\goodevent$. Let $r^*$ be a sufficiently
large integer such that $G_{\ge r^*} = \{O\}$. By Lemma
\ref{lem:effichain}, it holds that $\tilfam_{r^*+1} =
\subsetfam_{r^*+1}=\{O\}$ conditioning on $\goodevent_0$. Thus, the
condition at line~\ref{line:effiIf} is eventually satisfied, either
before or at round $r^*+1$.

It suffices to show that the condition of the if-statement at line
\ref{line:efficheck} is also met, and thus the algorithm would return
the correct answer, instead of reporting an error. Fix $k\in[r-1]$ and
$A\in\subsetfam$ with $\meanhat{k}{}(A) < \theta_k$. Since $A \notin
\subsetfam_{k+1}$, by Lemma~\ref{lem:effichain}, $A \notin G_{\ge k+1}$,
and therefore $A \in G_t$ for some $t \le k$. By Lemma
\ref{lem:effichain}, $A \in \tilfam_t$. Thus, conditioning on
$\goodevent$,
$$\hat\mu(O)-\hat\mu(A)
> \mu(O)-\mu(A)-\epsilon_t/\lambda
> \epsilon_{t+1}-\epsilon_t/\lambda
= (1/2-1/\lambda)\epsilon_t
\ge 2\epsilon_k/\lambda\text{.}$$
It follows that $\hat\mu(O)-\hat\mu(A)\ge2\epsilon_k/\lambda$ for all
$A\in\subsetfam$ with $\meanhat{k}{}(A) < \theta_k$. According to the
specification of $\checksol$, $\checksol(O, \meanhat{k}{}, \hat\mu,
\theta_k, \epsilon_k/\lambda)$ always returns true, and thus the
algorithm returns the optimal set $O$.

\paragraph{Soundness conditioning on $\goodevent$.} Now we show that the
algorithm never returns an incorrect answer (i.e., a sub-optimal set)
conditioning on $\goodevent$. Suppose that at some round $r$, $\tilfam_r
= \{\hat O\}$ for $\hat O\in\subsetfam\setminus\{O\}$, and thus the
condition at line~\ref{line:effiIf} is met. It suffices to show that the
algorithm reports an error, rather than incorrectly returning $\hat O$
as the answer.

Since the correct answer $O$ is not in $\tilfam_r$, there exists an
integer $k$ such that
$O\in\tilfam_k\setminus\tilfam_{k+1}$. Conditioning on $\goodevent$, it
holds that
$$\hat\mu(\hat O)-\hat\mu(O)
<\mu(\hat O)-\mu(O)+\epsilon_k/\lambda
\le \epsilon_k/\lambda\text{.}$$
Therefore, we have $\meanhat{k}{}(O) < \theta_k-\epsilon_k/\lambda$ and $\hat\mu(\hat O)-\hat\mu(O)\le\epsilon_k/\lambda$.
Thus $\checksol(\hat O, \meanhat{k}{}, \hat\mu, \theta_k,
\epsilon_k/\lambda)$ is guaranteed to return false, and the algorithm
does not return the incorrect answer $\hat O$.

This finishes the proof of Lemma~\ref{lem:efficorrect}.
\end{proof}

\subsection{Sample Complexity}
\label{sec:sample-comp-eff}

We restate Lemma~\ref{lem:effisample} for convenience.

\noindent\textbf{Lemma~\ref{lem:effisample}} (restated) \textit{
  For any $\delta\in(0,0.01)$ and \combibandit{} instance $\combiband$,
  $\algeffi(\combiband, \delta)$ takes
  $$O\left(\gklow(\combiband)\ln\delta^{-1} + \gklow(\combiband)\ln^2\Delta^{-1}\left(\ln\ln\Delta^{-1} + \ln|\subsetfam|\right)\right)$$
  samples conditioning on event $\goodevent_0\cap\goodevent$.
}

\begin{proof}[Proof of Lemma~\ref{lem:effisample}]
  For a \combibandit{} instance $\combiband = (S, \subsetfam)$, let
  $\tau^*$ be the optimal solution to the program in
  \eqref{eq:gklowdef}:
  \begin{equation*}\begin{split}
      \textrm{minimize}~~&\sum_{i\in S}\tau_i\\
      \textrm{subject to}~~&\sum_{i\in O\oplus A}1/\tau_i \le [\mu(O)-\mu(A)]^2,~\forall A\in\subsetfam\\
      & \tau_i > 0,~\forall i\in S\text{.}
    \end{split}\end{equation*}
  Recall that
  $\gklow(\combiband) = \sum_{i\in S}\tau^*_i\text{.}$
  
  We start by upper bounding the number of samples taken according to
  $\numsamp{r}{}$ in each round $r$. Specifically, we construct a
  small feasible solution to the mathematical program defined in
  $$\simest(
  \meanhat{k-1}{},
  \thetahi,
  \thetalo,
  \epsilon_k/\lambda,
  \delta_r
  )\text{,}$$
  where $\thetahi = \theta_{k-1}-\epsilon_{k-1}/\lambda$
  and $\thetalo = \theta_{k-1}-2\epsilon_{k-1}/\lambda$,
  thereby obtaining a bound on the optimal solution of the program.
  
  Let $\alpha = 64\lambda^2\ln(2/\delta_r)$ and $m_i =
  \alpha\tau^*_i$. Fix $A,
  B\in\{A'\in\subsetfam:\meanhat{k-1}{}(A')\ge\thetalo\}$. By Lemma~\ref{lem:effichain},
  we have $A, B \in G_{\ge k-2}$, and thus both
  $\mu(O)-\mu(A)$ and $\mu(O)-\mu(B)$ are smaller than or equal to
  $\epsilon_{k-2} = 4\epsilon_k$. It follows that
  \begin{equation*}\begin{split}
      \sum_{i\in A\oplus B}1/m_i
      &\le \alpha^{-1}\left(\sum_{i\in O\oplus A}1/\tau^*_i+\sum_{i\in O\oplus B}1/\tau^*_i\right)\\
      &\le \alpha^{-1}\left[[\mu(O)-\mu(A)]^2+[\mu(O)-\mu(B)]^2\right]\\
      &\le 2\alpha^{-1}\cdot(4\epsilon_{k})^2 = \frac{(\epsilon_k/\lambda)^2}{2\ln(2/\delta_r)}\text{.}
    \end{split}\end{equation*}
  Here the second step holds since $\tau^*$ is a feasible solution to
  the program in \eqref{eq:gklowdef}. The last step applies $\alpha =
  64\lambda^2\ln(2/\delta_r)$.
  Therefore, this setting $\{m_i\}$ is a valid solution even for the
  \textit{tightened} program defined just above the description of
  $\simest(\meanhat{k-1}{}, \thetahi, \thetalo, \epsilon_k/\lambda,
  \delta_r)$ (Algorithm~\ref{algo:simest}). Moreover, by our choice of
  $m_i = \alpha \tau_i^*$, the number of samples contributed by $r$ and
  $k$ is upper bounded by
  $$\sum_{i\in S}m_i
  = \alpha\sum_{i\in S}\tau^*_i
  = O(\gklow(\combiband)\ln\delta_r^{-1})
  = O\left(\gklow(\combiband)\left(\ln r +\ln|\subsetfam|\right)\right)\text{.}$$
  In sum, \algeffi{} takes $O\left(\gklow(\combiband)\left(r\ln r + r\ln|\subsetfam|\right)\right)$ samples in round $r$.

  This can now be used to bound the number of samples in all but the
  last round.  Let $\Delta = \mu(O) -
  \max_{A\in\subsetfam\setminus\{O\}}\mu(A)$ and $r^* =
  \left\lfloor\log_2\Delta^{-1}\right\rfloor+1$. Observe that $G_{\ge
    r^*}=G_{\ge r^*+1}=\{O\}$. Thus by Lemma~\ref{lem:effichain},
  $\tilfam_{r^*+1} = \{O\}$ and the algorithm terminates before or at
  round $r^*+1$. Summing over all $r$ between $1$ and $r^*$ yields
  \begin{equation*}\begin{split}
      O\left(\gklow(\combiband)\sum_{r=1}^{r^*}r\cdot\left(\ln
          r+\ln|\subsetfam|\right)\right)
      &=	O\left(r^*\cdot\gklow(\combiband)\left(r^*\ln r^* + r^*\ln|\subsetfam|\right)\right)\\
      &=
      O\left(\ln^2\Delta^{-1}\cdot\gklow(\combiband)\left(\ln\ln\Delta^{-1}+\ln|\subsetfam|\right)\right)\text{.}
    \end{split}\end{equation*}
  
  Then we bound the number of samples taken at the last round, denoted
  by round $r$. Let $\beta = 32\lambda^2\ln(2r|\subsetfam|/\delta)$, and
  $m_i = \beta \tau^*_i$. We show that $\{m_i\}$ is a feasible solution
  to the program in
  $$\verify(
  \{\meanhat{k}{}\},
  \{\thetahi_k\},
  \{\thetalo_k\},
  \delta/(r|\subsetfam|)
  )\text{.}$$
  Here $\thetahi_k = \{\theta_k - \epsilon_k/\lambda\}$,
  and $\thetalo_k = \{\theta_k - 2\epsilon_k/\lambda\}$.
    Fix $k\in[r]$ and
  $A\in\{A'\in\subsetfam:\meanhat{k-1}{}(A)\ge\thetalo_{k-1}\}$. By
  Lemma~\ref{lem:effichain}, we have $A\in G_{\ge k-2}$, which implies
  that $\mu(O)-\mu(A)\le\epsilon_{k-2}=4\epsilon_k$. Recall that by
  Lemma~\ref{lem:effisurvive}, $\hat O = O$. Thus we have
  \begin{equation*}\begin{split}
      \sum_{i\in \hat O\oplus A}1/m_i
      &=	\beta^{-1}\sum_{i\in O\oplus A}1/\tau^*_i\\
      &\le\beta^{-1}[\mu(O)-\mu(A)]^2\\
      &\le32\beta^{-1}\epsilon_k^2 = \frac{(\epsilon_k/\lambda)^2}{2\ln(2r|\subsetfam|/\delta)}\text{.}
    \end{split}\end{equation*}
  Recall that $r \le r^*+1 = O(\ln\Delta^{-1})$. Therefore, the number of samples taken in the last round is upper bounded by
  \begin{equation*}\begin{split}
      \sum_{i\in S}m_i = \beta\sum_{i\in S}\tau^*_i
      &= O\left(\gklow(\combiband)\left(\ln\delta^{-1} + \ln r + \ln|\subsetfam|\right)\right)\\
      &= O\left(\gklow(\combiband)\left(\ln\delta^{-1} + \ln\ln\Delta^{-1} + \ln|\subsetfam|\right)\right)\text{.}
    \end{split}\end{equation*}
    Therefore, conditioning on $\goodevent_0\cap\goodevent$, the number of
  samples taken by \algeffi{} is
  $$O\left(\gklow(\combiband)\ln\delta^{-1} +
    \gklow(\combiband)\ln^2\Delta^{-1}\left(\ln\ln\Delta^{-1} +
      \ln|\subsetfam|\right)\right)\text{.}$$
  This completes the analysis of the sample complexity.
\end{proof}

\section{Worst-Case Lower Bound for Combinatorial Bandit}
\label{sec:lower_bound_comb}

In this section we construct a family of \combibandit instance to show that
$$
O(\gklow(\combiband) \cdot  \left( \ln |\subsetfam| + \ln\delta^{-1}\right) ) 
$$
samples are required for any $\delta$-correct algorithm in the worst case. 
We need the following lemma for our theorem, which constructs a list of
subsets resembling the Nisan-Wigderson design~\cite{nisan1994hardness}.
\begin{Lemma}\label{lm:random-exist}
  Given an integer $n$ and there exists a list of $m = 2^{cn}$ subsets
  $S_1,S_2,\dotsc,S_m$ of $[n]$ where $c$ is a universal constant, such
  that $|S_i| = \ell = \Omega(n)$ for each $S_i$, and $|S_i \cap S_j|
  \le \ell/2$ for each $i \ne j$.
\end{Lemma}
\begin{proof}
  We prove the lemma via the probabilistic method.  Let $\ell = n/10$,
  and $m = 2^{cn}$. We simply let $S_1,S_2,\dotsc,S_m$ be a sequence of
  independent uniformly random subsets of $[n]$ with size
  $\ell$. Clearly, we have
  $$
  \Pr[ |S_i \cap S_j| > \ell/2] \le 2^{-\Omega(n)}
  $$
  for each $i \ne j$. Hence, we can set the constant $c$ to be
  sufficiently small so that
  $$
  \Pr[\exists i \ne j, |S_i \cap S_j| > \ell /2 ] < m^2 \cdot
  2^{-\Omega(n)} <1,
  $$
  which implies the existence of the desired list.
\end{proof}
We now prove Theorem~\ref{theo:worst-case-lowb-comb}, which we restate 
here for convenience.

\noindent\textbf{Theorem~\ref{theo:worst-case-lowb-comb}.} (restated)
\textit{ (i) For $\delta \in (0,0.1)$, two positive integers $n$ and $m
  \le 2^{c n}$ where $c$ is a universal constant, and every
  $\delta$-correct algorithm $\alg$ for \combibandit, there exists an
  infinite sequence of $n$-arm instances
  $\combiband_1=(S_1,\subsetfam_1),\combiband_2=(S_2,\subsetfam_2),\dotsc,$
  such that $\alg$ takes at least
  $$
  \Omega( \gklow(\combiband_k) \cdot (\ln |\subsetfam_k| + \ln \delta^{-1}))
  $$
  samples in expectation on each $\combiband_k$, $|\subsetfam_k| = m$ for all $k$, and $\gklow(\combiband_k)$ goes to infinity. \\
  (ii) Moreover, for each $\combiband_k$, there exists a
  $\delta$-correct algorithm $\alg_k$ for \combibandit such that
  $\alg_k$ takes
  $$
  O(\gklow(\combiband_k) \cdot \operatorname{poly}(\ln n,\ln\delta^{-1}))
  $$
  samples in expectation on it.
  (The constants in $\Omega$ and $O$ do not depend on $n,m,\delta$ and $k$.)
}

Our proof for the first part is based on a simple but delicate reduction
to the problem of distinguishing two instances with a much smaller
confidence parameter $O(\delta/|\subsetfam|)$.

\begin{proofof}{the first part of Theorem~\ref{theo:worst-case-lowb-comb}}
  We fix a real number $\Delta \in (0,0.1)$, and let constant $c$ and
  $\ell = \Omega(n)$ be as in Lemma~\ref{lm:random-exist}. For each
  subset $A \subseteq [n]$, we define $\combiband_{A}$ to be the $n$-arm
  instance whose $i$-th arm has mean $\Delta$ when $i \in A$ and mean
  $0$ otherwise. Let $S_1,S_2,\dotsc,S_{2^{cn}}$ be a list whose
  existence is guaranteed by Lemma~\ref{lm:random-exist}, and set
  $\subsetfam = \{S_1,S_2,\dotsc,S_m\}$.

  Let $\alg$ be a $\delta$-correct algorithm for \combibandit. For a
  subset $A \in \subsetfam$, let $\event_{A}$ be the event that $\alg$
  outputs 
  $\combiband_A$. Fixing a subset $A \in \subsetfam$, the definition
  implies that 
  $$
  \sum_{B \in \subsetfam,B \ne A} \Pr_{\alg,\combiband_{A}}[\event_{B}]
  \le \delta.
  $$
  By a simple averaging argument, there exists another subset $B \in
  \subsetfam$ such that $\Pr_{\alg,\combiband_A}[\event_{B}] \le
  2\delta/ |\subsetfam|$.  Now, from the fact that $\alg$ is
  $\delta$-correct, we have $\Pr_{\alg,\combiband_B}[\event_{B}] \ge
  0.9$.  Combining the above two facts with Lemma~\ref{lem:CoD}, we can
  see that $\alg$ must spend at least
  $$
  d\left(\Pr_{\alg,\combiband_B}[\event_{B}],\Pr_{\alg,\combiband_A}[\event_{B}]\right)
  \cdot \Delta^{-2} = \Omega((\ln |\subsetfam| + \ln \delta^{-1})
  \cdot \Delta^{-2})
  $$
  samples on $\combiband_B$ in expectation.  On the other hand, one can
  easily verify that setting $\tau_i = \Theta(1/\ell \cdot \Delta^{-2})$
  satisfies the constraints in the lower bound
  program~\eqref{eq:gklowdef} in Section~\ref{sec:instance-lowb}, and
  hence $\gklow(\combiband_B) \le \Theta(n/\ell \cdot \Delta^{-2}) =
  \Theta(\Delta^{-2})$.  

  Therefore, to prove the first part of this theorem, we set $\Delta$ to
  be $1/n,1/2n,1/3n,\dotsc$ and set $\combiband_k$ to be corresponding
  $\combiband_B$ constructed from the above procedure. (The
    property that $\Delta \le 1/n$ will be used in the proof for the
    second part.)
\end{proofof}

For the second part of Theorem~\ref{theo:worst-case-lowb-comb}, we first
design an algorithm for an interesting special case of $\generalbandit$.

\begin{theo}\label{theo:ball-case}
  For a positive integer $n$, a positive real number $r \le 1$ and a
  vector $u \in \mathbb{R}^n$, we define
  $$
  \anssetcol = \{ \{u\}, \{ v \in \mathbb{R}^n : \|u-v\|_2 \ge r \}  \}.
  $$
  There is a $\delta$-correct algorithm for $\generalbandit$ which takes
  $$
  O\left(n \ln^2 n \cdot r^{-2} \cdot (\ln n + \ln \delta^{-1}) \cdot
    \ln\delta^{-1}\right)
  $$
  samples in expectation on the instance $\inst = (S, \anssetcol)$,
  where $S$ is a sequence of arms with mean profile $u$.
\end{theo}

Before proving Theorem~\ref{theo:ball-case}, we show it implies the
moreover part of Theorem~\ref{theo:worst-case-lowb-comb}.

\begin{proofof}{the moreover part of Theorem~\ref{theo:worst-case-lowb-comb}}
  Let $\combiband_k=(S_k,\subsetfam_k)$ be a constructed instance in the
  proof of the first part of Theorem~\ref{theo:worst-case-lowb-comb},
  and $\Delta$, $B$, $\subsetfam$, $\ell$, $m$ be the corresponding
  parameters during its construction. From our choice of $\Delta$, we
  have $\Delta \le 1/n$. And in the whole proof we assume $n$ is
  sufficiently large for simplicity.

  \newcommand{\algball}{\alg_{\mathsf{ball}}}

  Let $\algball$ be the algorithm guaranteed by
  Theorem~\ref{theo:ball-case}. Our algorithm works as follows:
  \begin{itemize}
  \item Given an instance $\combiband=(S,\subsetfam)$, run an arbitrary
    $\delta$-correct algorithm for \combibandit when $\subsetfam \ne
    \subsetfam_k$.
	
  \item Run $\algball$ with $r = c_1 \cdot \Delta \cdot \sqrt{n}$, mean
    profile $u$ set as the mean profile of instance $\combiband_k$ and
    confidence level set as $\delta/2$, where $c_1$ is a constant to be
    specified later. (Note that $r \le 1$ as $\Delta \le 1/n$.)
	
    \begin{itemize}
    \item Recall that $\anssetcol = \{A_1,A_2\}$, where $A_1 =\{u\}$ and
      $ A_2 = \{ v \in \mathbb{R}^n : \|u-v\|_2 \ge r \}$.
    \item (Case I) If $\algball$ returns $A_1$, outputs set $B$. 
    \item (Case II) Otherwise, run an arbitrary $\delta/2$-correct
      algorithm for \combibandit, and outputs its output.
    \end{itemize}
  \end{itemize}
	
  \newcommand{\eventg}{\event_{\mathsf{good}}}
	
  First, we condition on the event $\eventg$ that both $\algball$ and
  the simulated algorithm in Case II operate correctly, which happens
  with probability at least $1-\delta$.  Then we prove its
  correctness. Since we condition on $\eventg$, whenever it enters Case
  II, it must output the correct answer. So it would only make mistakes
  in Case I.
	
  Now we suppose that the algorithm enters Case I. Let $u$ be the mean
  profile of the instance $\combiband_k$, and $v$ be the mean profile of
  the given instance $\combiband$. Conditioning on $\eventg$, from
  Theorem~\ref{theo:ball-case}, we must have
  \begin{equation}\label{eq:bound-l2}
    \|u-v\|_2 < c_1 \cdot \Delta \cdot \sqrt{n},
  \end{equation}
  since otherwise $v \in A_2$ and $\algball$ would output $A_2$ instead.
  We are going to show in this case, the correct answer must be
  $B$. That is tantamount to prove that for $A \in \subsetfam$ with $A
  \ne B$, we have
  $$
  \sum_{i \in B} v_i - \sum_{i \in A} v_i = \sum_{i \in B \setminus A}
  v_i - \sum_{i \in A \setminus B} v_i > 0.
  $$
  Note that from the construction of $\combiband_k$, we have $u_i =
  \Delta$ when $i \in B$, and $u_i = 0$ otherwise. Therefore,
  $$
  \sum_{i \in B} u_i - \sum_{i \in A} u_i = \sum_{i \in B\setminus A}
  u_i - \sum_{i \in A \setminus B} u_i \ge (|B| - |A\cap B|) \cdot
  \Delta = \Delta\cdot \ell/2,
  $$
  and
  \begin{align*}
    \bigg| \bigg(\sum_{i \in B} u_i - \sum_{i \in A} u_i\bigg) - \bigg(\sum_{i \in B} v_i - \sum_{i \in A} v_i \bigg) \bigg|
    &=\bigg| \sum_{i \in B\setminus A} (u_i-v_i) - \sum_{i \in A \setminus B} (u_i-v_i)  \bigg|\\
    &\le \|u-v\|_1 \le \|u-v\|_2 \cdot \sqrt{n} < c_1 \cdot \Delta \cdot n.
  \end{align*}
  Since $\ell = \Omega(n)$, we set $c_1$ to be a sufficiently small
  constant so that $c_1 \cdot \Delta \cdot n < \Delta \cdot \ell
  /2$. This implies that
  $$
  \sum_{i \in B \setminus A} v_i - \sum_{i \in A \setminus B} v_i > 0
  $$
  and hence
  $$
  \sum_{i \in B} v_i > \sum_{i \in A} v_i.
  $$
  Therefore, $B$ strictly dominates any other set $A \in \subsetfam$ in
  the given instance $\combiband$, which means it is the correct
  answer.This concludes the proof for its correctness.
	
  For the sample complexity on the instance $\combiband_k$, note that
  conditioning on $\eventg$, it must enter Case I, which means it takes
  \begin{align*}
    &O(n \cdot r^{-2} \cdot \operatorname{poly}(\ln n,\ln \delta^{-1})) \\
    &= O(\Delta^{-2} \cdot \operatorname{poly}(\ln n,\ln \delta^{-1})) \\
    &= O(\gklow(\combiband_k) \cdot \operatorname{poly}(\ln n,\ln
    \delta^{-1}))
  \end{align*}
  samples on that instance with probability $1-\delta$.
	
  Strictly speaking, the above sample complexity does not hold in
  expectation, as the algorithm may takes an arbitrary number of samples
  if $\eventg$ does not happen. So we complete the final step by a
  simple application of the parallel simulation trick (see
  Lemma~\ref{lem:parasim}), to transform the above algorithm into an
  algorithm with an
  $$
  O(\gklow(\combiband_k) \cdot \operatorname{poly}(\ln n,\ln \delta^{-1})) 
  $$
  expected sample complexity on $\combiband_k$.
\end{proofof}

Finally, we devote the rest of this section to prove
Theorem~\ref{theo:ball-case}.

\begin{proofof}{Theorem~\ref{theo:ball-case}}
  Without loss of generality, we can assume that $u$ is the all zero
  vector $\mathbf{z}$. Let $A_1 = \{u\}$ and $A_2 = \{ v \in
  \mathbb{R}^n : \|u-v\|_2 \ge r \}$, then $\anssetcol =
  \{A_1,A_2\}$. For simplicity We assume $n$ is sufficiently large in
  the whole proof. Our algorithm works as follows.
  \begin{itemize}
  \item Given an instance $\inst_0 = (S,\anssetcol_0)$, when
    $\anssetcol_0 \ne \anssetcol$, run another $\delta$-correct
    algorithm for \generalbandit instead (for example the algorithm in
    Section~\ref{sec:general-algo}).
  \item For each integer $k$ from $1$ to $\lceil\log_2 n \rceil + 2$.
    \begin{itemize}
    \item Pick $N_k = c_1\cdot n \ln n \cdot 2^{-k} \cdot
      \ln\delta^{-1}$ arms at uniformly random, let the set of taken
      arms be $S_k$.
    \item For each arm $a \in S_k$, take $c_2 \cdot r^{-2} \cdot 2^{k}
      \cdot (\ln n + \ln \delta^{-1} )$ samples from it and let
      $\hamean{a}^k$ be its empirical mean. If there is an arm $a \in
      S_k$ such that $|\hamean{a}^k| > r \cdot 2^{-k/2-1}$, output $A_2$
      and terminates the algorithm.
    \end{itemize}
  \item If the algorithm does not halt in the above step, then output
    $A_1$.
  \end{itemize}
	
  \newcommand{\eventg}{\event_{\mathsf{good}}}

  To show the algorithm works, we start by setting $c_2$ to be a
  sufficiently large constant so that for each $k$, and each arm $a \in
  S_k$, we have
  $$
  \Pr\left[ | \hamean{a}^k - \amean{a} | \ge r \cdot 2^{-k/2-1} \right]
  < \delta/2 \cdot n^{-2}.
  $$
  By a simple union bound over all $k$, with probability at least
  $1-\delta/2$,
  $$
  \left|\hamean{a}^k - \amean{a} \right| < r \cdot 2^{-k/2-1}
  $$ 
  for all $k$ and $a \in S_k$. We denote the above as event $\eventg$.
  A simple calculation shows that the above algorithm takes
  $$
  O\left(n \ln^2 n \cdot r^{-2} \cdot (\ln n + \ln \delta^{-1}) \cdot
    \ln\delta^{-1} \right)
  $$
  samples.
	
  Next we prove its correctness. Let the mean profile of the given
  instance $\inst_0$ be $x$.  We first show that when $x$ equals
  $u=\mathbf{z}$ (i.e. $x \in A_1$), the algorithm outputs $A_1$ with
  probability at least $1-\delta$. Conditioning on event $\eventg$, for
  each $k$ and $a \in S_k$, we have $|\hamean{a}^k| < r \cdot
  2^{-k/2-1}$, therefore the algorithm outputs the correct answer
  $A_1$. Since $\Pr[\eventg] \ge 1-\delta$, we finish the case when $x =
  \mathbf{z}$.
	
  \newcommand{\Wsum}{\mathsf{W}}
	
  For the second case when $x \in A_2$, we have $\|x\|_2 \ge r$, which
  means
  $$ 
  \sum_{i=1}^{n} x_i^2/r^2 \ge 1.
  $$
  Now, for each positive integer $k$, we define $X_k = \{ i : x_i^2/r^2
  \in (2^{-k},2^{-k+1}] \}$, and let
  $$
  \Wsum(X_k) := \sum_{i \in X_k} x_i^2/r^2.
  $$ 
  We can see
  $$
  \sum_{k=\lceil \log_2 n \rceil + 3}^{\infty} \Wsum(X_k) \le n \cdot
  2^{-\lceil\log_2 n\rceil - 2} \le 1/4.
  $$
  and hence
  $$
  \sum_{k=1}^{\lceil \log_2 n \rceil +2} \Wsum(X_k) \ge 3/4.
  $$
	
  \newcommand{\kstar}{k^{\star}}
  \newcommand{\eventgg}{\event_{\mathsf{non\text{-}empty}}}
	
  Let $\kstar$ be the $k$ with maximum $\Wsum(X_{k})$, then we have 
  $$
  \Wsum(X_{\kstar}) \ge \frac{1}{2 \log_2 n}.
  $$
  When $k = \kstar$ in the above algorithm, we have that $|X_{k}| \ge
  \frac{1}{2 \log_2 n} \cdot 2^{k-1}$. Therefore, we can set $c_1$ to be
  sufficiently large so that we have
  $$
  \Pr\left[|S_k \cap X_{k}| > 0 \right] > 1-\delta/2.
  $$
  Let the above be event $\eventgg$. We claim that conditioning on both
  $\eventg$ and $\eventgg$, our algorithm correctly outputs $A_2$. Let
  $a \in S_k \cap X_{k}$, we have that $|\amean{a}| > r \cdot
  2^{-k/2}$. Moreover, $|\hamean{a}^k| > r \cdot 2^{-k/2-1}$ from the
  definition of $\eventg$. Hence our algorithm outputs $A_2$. Since
  $\Pr[\eventg \cap \eventgg] \ge 1-\delta$, this finishes the case when
  $x \in A_2$, and hence the proof.
\end{proofof}


\section{Another Worst-Case Lower Bound for the General Case}
\label{sec:anotherlb}

Recall that \combibandit\ is clearly a special case of \generalbandit{},
so the lower bound in Section~\ref{sec:lower_bound_comb} also applies to
the latter problem.  Here we present another lower bound for the
\generalbandit{} problem, which illustrates the ``non-uniform'' nature
of the instance-wise lower bound $\gklow(\inst)$. In the following we
will construct a family of instances which are similar to an
$\mathsf{OR}$ function and prove an
$$
O(\gklow(\inst) \cdot (n + \ln \delta^{-1}))
$$
worst-case lower bound for all $\delta$-correct algorithm $\alg$ for the
general sampling problem.

\noindent\textbf{Theorem~\ref{theo:lower-bound-general}.} (restated)\emph{
  For $\delta \in (0,0.1)$, a positive integer $n$ and every
  $\delta$-correct algorithm $\alg$ for the general sampling problem,
  there exists an infinite sequence of $n$-arm instances
  $\inst_1=(S_1,\anssetcol_1),\inst_2=(S_2,\anssetcol_2),\dotsc,$ such
  that $\alg$ takes at least
  $$
  \Omega( \gklow(\inst_k) \cdot (\ln \delta^{-1} + n))
  $$
  samples in expectation on each $\inst_k$, $|\anssetcol_k| =O(1)$ for
  all $k$, and $\gklow(\inst_k)$ goes to infinity. 
 Moreover, for each
  $\inst_k$, there exists a $\delta$-correct algorithm $\alg_k$ for
  \generalbandit such that $\alg_k$ takes
  $$
  O(\gklow(\inst_k) \cdot \ln\delta^{-1})
  $$
  samples in expectation on it.  (The constants in $\Omega$ and $O$ does
  not depend on $n,m,\delta$ and $k$.)  } 

\newcommand{\zero}{\mathbf{z}}
\newcommand{\basis}{\mathbf{e}}

\begin{proof}
  We fix a real number $\Delta \in (0,1]$. 
  Let $\zero$ be the all zero
  vector with length $n$, and $\basis_i$ be the length-$n$ vector whose
  $i$-th element is $\Delta$ and all other elements are zero.
	  Consider the following collection of answer sets $\anssetcol =
  \{A_1,A_2\}$, where $A_1 = \{ \basis_1,\basis_2,\dotsc,\basis_n \}$
  and $A_2 = \{ \zero \}$. That is, we must distinguish between the
  cases when all arms have mean zero, and when exactly one arm has mean
  $\Delta$.  In the rest of the proof, we will always assume the
  collection of answers of the instances are $\anssetcol$. Therefore, to
  specify an instance, we only need to specify a mean profile.
	
  For each $i \in [n]$, let $\inst^{(i)}$ be the instance with mean profile
  $\basis_i$, and $\inst^{(0)}$ be the instance with mean profile
  $\zero$. Let $\alg$ be a $\delta$-correct algorithm for the general
  sampling problem.  First, it is not hard to see that $\gklow(\inst^{(i)})
  = \Delta^{-2}$ for each $i \in [n]$. We are going to show that $\alg$
  must draw at least $\Omega(n \cdot \Delta^{-2})$ samples in
  expectation on at least one $\inst^{(i)}$, where $i \in [n]$.  When $n$ is
  a constant, the above holds trivially, so we assume from now on that
  $n \ge 100$.
	
  \newcommand{\algnew}{\alg_{\mathsf{new}}}
	
  Consider the following new algorithm $\algnew$, which simply simulates
  $\alg$ as long as it draws at most $c_1 \cdot n \Delta^{-2}$ samples,
  where $c_1$ is a small constant to be specified later. $\algnew$
  outputs $\alg$'s output if $\alg$ halts before the specified amount of
  steps, and outputs $\perp$ otherwise.  Now, consider running $\algnew$
  on instance $\inst^{(0)}$. Let $p_{\mathsf{\perp}}$ be the probability
  that $\algnew$ outputs $\perp$, and $\tau_i$ be the number of samples
  taken on the $i$-th arm.   We claim that $p_{\mathsf{\perp}} > 0.5$.
	
  \newcommand{\istar}{{i^{\star}}}
  \newcommand{\icirc}{{i^{\circ}}}
  \newcommand{\eventerr}{\event_{\mathsf{err}}}

  Suppose for contradiction that $p_{\mathsf{\perp}} \le 0.5$. So with
  probability at least $0.5$, the simulated version of $\alg$ within
  $\algnew$ outputs something before it halts. Let $p_i$ be the
  probability that $\algnew$ outputs $i$. Since $\sum_{i=1}^n p_i \le
  1$, there are at least $n/2$ values of $i$ satisfying $p_i \le
  2/n$. Let $\icirc$ be such an $i$ with the minimum
  $\Ex_{\algnew,\inst^{(0)}}[\tau_{i}]$. We have $$
  \Ex_{\algnew,\inst^{(0)}}[\tau_\icirc] \le 2 \cdot c_1 \Delta^{-2}
  $$ 
  since 
  $$
  \sum_{i=1}^{n} \Ex_{\algnew,\inst^{(0)}}[\tau_i] \le c_1 \cdot n\Delta^{-2}.
  $$
  Let $\eventerr$ be the event that $\algnew$ outputs something
  different from $\perp$ and $\icirc$. Observe that
  $\Pr_{\algnew,\inst^{(0)}}[\eventerr] \ge 0.5 - 2/n \ge 0.48$.
		
  Now we run $\algnew$ on $\inst^{(\icirc)}$. Note that
  $\inst^{(\icirc)}$ and $\inst^{(0)}$ differ only on arm $\icirc$, so by
  Lemma~\ref{lem:CoD}, we have
  $$
  d\left(\Pr_{\algnew,\inst^{(0)}}[\eventerr],\Pr_{\algnew,\inst^{(\icirc)}}[\eventerr]\right)
  \le \Ex_{\algnew,\inst^{(0)}}[\tau_\icirc]\, \Delta^{2} = 2c_1.
  $$
  For a sufficiently small $c_1$, we can see the above leads to
  $\Pr_{\algnew,\inst^{(\icirc)}}[\eventerr] > 0.2$, which implies that
  running the original algorithm $\alg$ yields an incorrect answer with
  probability at least $0.2$ on instance $\inst^{(\icirc)}$, contradiction
  to the fact that $\alg$ is $\delta$-correct.
		
  \newcommand{\eventperp}{\event_{\perp}}
		
  Therefore, we conclude that $p_{\mathsf{\perp}} \ge 0.5$, which means
  the simulated $\alg$ runs for a full $c_1 \cdot n \Delta^{-2}$ period
  with probability at least $0.5$. Let $\eventperp$ be the event that
  $\algnew$ outputs $\perp$, and $\istar$ be the $i$ with minimum
  $\Ex_{\algnew,\inst^{(0)}}[\tau_i]$. Clearly, we have
  $\Ex_{\algnew,\inst^{(0)}}[\tau_i] \le c_1 \cdot \Delta^{-2}$.
  Again as above, we run $\algnew$ on instance $\inst^{(\istar)}$ and use
  Lemma~\ref{lem:CoD} to get
  $$
  d\left(\Pr_{\algnew,\inst^{(0)}}[\eventperp],\Pr_{\algnew,\inst^{(\istar)}}[\eventperp]\right)
  \le \Ex_{\algnew,\inst^{(0)}}[\tau_\istar]\Delta^{2} = c_1.
  $$
  
  To prove the first part of our theorem, the above discussion gives us
  the $\Omega(\gklow(\inst_k)\cdot n)$ term. The $\Omega(\gklow(\inst_k)
  \cdot \ln \delta^{-1})$ part follows from Theorem~\ref{theo:genlb}, and we can
  set $\Delta = 1/1,1/2,\dotsc,1/k$ and let $\inst_k$ be the
  corresponding $\inst_{\istar}$.

  For the second part of the theorem, let $\inst_k$ be a constructed
  instance, and $\Delta$, $\istar$ be the parameters as in its
  construction process. Our algorithm $\alg_k$ simply takes
  $O(\Delta^{-2} \ln \delta^{-1})$ samples from arm $\istar$ so that
  $$
  \Pr[ |\amean{\istar} - \hamean{\istar} | < \Delta/2] \ge 1-\delta/2,
  $$
  where $\hamean{\istar}$ is the empirical mean of arm $\istar$.
	
  If $\hamean{\istar} > \Delta/2$ then it outputs $A_1$ and halts. Else,
  it runs another $\delta/2$-correct algorithm for \generalbandit (for
  example, the algorithm in Section~\ref{sec:general-algo}).  Clearly,
  this is a $\delta$-correct algorithm. And with probability at least
  $1-\delta$, it takes $O(\Delta^{-2} \ln \delta^{-1})$ samples in total
  when running on instance $\inst_k$. Finally we can turn the sample
  complexity into a bound in expectation via the parallel simulation trick
  (Lemma~\ref{lem:parasim}), which concludes the proof.
\end{proof}


\end{document}